\documentclass[12pt]{amsart}
\usepackage[margin=1in]{geometry}
\usepackage{newtxtext}
\usepackage{newtxmath}
\usepackage{amsmath}
\usepackage{amscd}
\usepackage[alphabetic]{amsrefs}
\usepackage{graphicx}
\usepackage{caption}
\usepackage{tikz}
	\usetikzlibrary{decorations.pathreplacing}
	\usetikzlibrary{decorations.pathmorphing}
	\usetikzlibrary{patterns}
\usepackage{setspace}
\usepackage[hidelinks]{hyperref}
\usepackage{comment}
\usepackage{subcaption}
\usepackage{turnstile}

\usepackage{verified-badges}
\usepackage[T1]{fontenc}

\newcommand{\N}{\mathbb{N}}

\newtheorem{theorem}{Theorem}
\usepackage{etoolbox}
\usepackage{xparse}

\let\oldtheorem\theorem
\let\oldendtheorem\endtheorem
\RenewDocumentEnvironment{theorem}{o}{%
  \IfNoValueTF{#1}{\oldtheorem}{\oldtheorem[#1]\addcontentsline{toc}{subsection}{Theorem \thetheorem{} (#1)}}%
}{%
  \oldendtheorem
}

\theoremstyle{definition}

\theoremstyle{remark}

\newtheorem{remark}{Remark}

\usepackage{listings}

\lstdefinelanguage{Lean}{
  sensitive=true,
  morecomment=[l]{--},            
  morecomment=[s]{/-}{-/},        
  morestring=[b]",               
  alsoletter={'.},               
  keywords=[1]{                  
    theorem, lemma, def, abbrev, example,
    inductive, structure, class, instance,
    namespace, section, end, open,
    variable, variables, parameter, parameters,
    universe, universes,
    axioms, constant, constants,
    import, exporting,
    where, match, with, fun, by,
    if, then, else,
    let, in,
    do, return,
    deriving,
    macro, syntax, notation, infix, infixl, infixr, prefix, postfix,
    private, protected, partial, mutual,
    attribute, builtin,
    simp, dsimp, intro, intros, exact, apply, refine, constructor,
    cases, induction, revert, have, show, suffices,
    calc, conv, rfl, simp_all, aesop
  },
  keywords=[2]{                  
    Prop, Type, Sort, forall, exists
  },
  keywords=[3]{                  
    simp, dsimp, ring, omega, linarith, nlinarith
  }
}

\lstdefinestyle{lean}{
  language=Lean,
  basicstyle=\ttfamily\footnotesize,
  columns=fullflexible,
  keepspaces=true,
  showstringspaces=false,
  upquote=true,
  commentstyle=\itshape\color{gray!70},
  stringstyle=\color{orange!70!black},
  keywordstyle=[1]\bfseries\color{blue!65!black},
  keywordstyle=[2]\bfseries\color{purple!60!black},
  keywordstyle=[3]\bfseries\color{teal!60!black},
  frame=single,
  framerule=0.3pt,
  rulecolor=\color{black!20},
  backgroundcolor=\color{black!2},
  numbers=left,
  numberstyle=\tiny\color{black!45},
  numbersep=6pt,
  breaklines=true,
  breakatwhitespace=true,
  tabsize=2,
  literate=
    {∀}{{$\forall$}}1
    {∃}{{$\exists$}}1
    {→}{{$\to$}}1
    {α}{{$\alpha$}}1
    {∧}{{$\wedge$}}1
}

\newcommand{\leaninline}[1]{\lstinline[style=lean]!#1!}

\begin{document}

\title{Compression is all you need: Modeling Mathematics}
\date{\today}
\author{Vitaly Aksenov}
\address{\hspace{-\parindent}Vitaly Aksenov, Logical Intelligence}
\email{aksenov@logicalintelligence.com}

\author{Eve Bodnia}
\address{\hspace{-\parindent}Eve Bodnia, Logical Intelligence}
\email{evebodnia@logicalintelligence.com}

\author{Michael H. Freedman}
\address{\hspace{-\parindent}Michael H. Freedman, Logical Intelligence; Center of Mathematical Sciences and Applications, Harvard University, Cambridge, MA 02138, USA}
\email{michael.freedman@logicalintelligence.com, mfreedman@cmsa.fas.harvard.edu}

\author{Michael Mulligan}
\address{\hspace{-\parindent}Michael Mulligan, Logical Intelligence; Department of Physics and Astronomy, University of California, Riverside, CA 92521, USA}
\email{michael.mulligan@logicalintelligence.com, michael.mulligan@ucr.edu}

\maketitle


\begin{abstract}
Human mathematics (HM), the mathematics humans discover and value, is a vanishingly small subset of formal mathematics (FM), the totality of all valid deductions.
We argue that HM is distinguished by its compressibility through hierarchically nested definitions, lemmas, and theorems.
We model this with monoids.
A mathematical deduction is a string of primitive symbols; a definition or theorem is a named substring or macro whose use compresses the string.
In the free abelian monoid $A_n$, a logarithmically sparse macro set achieves exponential expansion of expressivity.
In the free non-abelian monoid $F_n$, even a polynomially-dense macro set only yields linear expansion; superlinear expansion requires near-maximal density.
We test these models against MathLib, a large Lean~4 library of mathematics that we take as a proxy for HM.
Each element has a depth (layers of definitional nesting), a wrapped length (tokens in its definition), and an unwrapped length (primitive symbols after fully expanding all references).
We find unwrapped length grows exponentially with both depth and wrapped length; wrapped length is approximately constant across all depths.
These results are consistent with $A_n$ and inconsistent with $F_n$, supporting the thesis that HM occupies a polynomially-growing subset of the exponentially growing space FM.
We discuss how compression, measured on the MathLib dependency graph, and a PageRank-style analysis of that graph can quantify mathematical interest and help direct automated reasoning toward the compressible regions where human mathematics lives.
\end{abstract}

\section{Introduction}

In this paper, we argue that math is soft and squishy---that this is its defining characteristic.
By ``math'' we do not mean the totality of all possible formal deductions, formal mathematics (FM), but rather human mathematics (HM), the type of arguments humans find and those we will appreciate when our AI agents find them for us.\footnote{However mathematics is formalized, we know from G\"odel and other sources that there will be true statements without proofs (such as consistency of the system).
It is possible that extremely simple $\Pi^0_1$ statements of Peano arithmetic, such as the Goldbach conjecture (GC: ``Every even number ${}>2$ is the sum of two primes''), could be both true and without any proof.
Since our discussion is anchored to the concept of proof, we would not count GC as part of HM or even FM if that is the case, despite the fact that it is of interest to humans.
A more subtle question: suppose GC has a proof but the shortest one is $10^{100}$ lines long---is it part of HM\@?
Fortunately we do not have to adjudicate this; our results are not sharp enough to require us to identify an exact frontier to HM\@.
Moreover, HM might more precisely be considered a measure on FM that fades out rather than abruptly terminating at the edge of a subset; see~\cite{BDF}.}
By ``soft and squishy,'' we mean compressible through the use of hierarchically nested concepts: definitions, lemmas, and theorems.

The finding that math is about compression is not new.
We were scooped 3,000 years ago by the invention of place notation.
Consider $\mathbb{N}$, the natural numbers, with generating set $\{1\}$.
Place notation introduces additional symbols, or \emph{macros}: ``10'' for ten ones, ``100'' for ten tens, and so on.
With logarithmically many macros, expressivity expands exponentially.
This (exponential) expansion of expressivity is the flip side of notational compression.
Creating and exploiting definitions expands what we can reach by compressing expressions written in a primitive, definition-poor language.
(Our theorems below will be stated in terms of expansion; our informal discussions often use compression.)

Formal mathematics can be viewed as a directed hypergraph (DH) emerging from axioms and syntactical rules \cite{BDF}.
The DH records the full deduction space: every possible proof step, with each hyperedge specifying which premises are combined (Fig.~\ref{fig:dh-vs-dag}, left).
A proof is a sub-hypergraph of the DH; flattened into a linear sequence, it becomes a string of primitive symbols.
We study finitely-generated monoids as models for such strings: word length measures size, and naming a substring for reuse---a macro---compresses it.
(Monoids with relations can simulate Turing machines \cite{Post}, so, despite its simplicity, this basic framework is computationally universal.)
The simplest case is $A_1 = \mathbb{N}$, the natural numbers. 
To study compression more generally, we consider the free abelian monoid $A_n$ and the free (non-abelian) monoid $F_n$ (with $n$ denoting the number of generators).
In $A_n$, the generators commute, so only the multiplicities matter.
In $F_n$, the order of generators is important, and there are no relations; since formal proofs are strings of symbols where order matters, $F_n$ might be presumed to model formal deduction.
We will argue, contrary this expectation, that the compression exhibited by human mathematics is characteristic of $A_n$, not $F_n$.

Our main theoretical results quantify the expansion that macros achieve in $A_n$ and $F_n$.
In $A_n$, logarithmically many macros achieve exponential expansion (Theorem~\ref{thm:abelian-place}), and macros of polynomial density (growth exponent $1/k$) can yield infinite expansion---every element expressible with bounded length---via Waring's theorem (Theorem~\ref{thm:abelian-waring}).
In $F_n$ even polynomially growing macros (polynomial as a function of radius, i.e., polylogarithmic as a function of volume) yield only linear expansion (Theorem~\ref{thm:free-polynomial}).
Superlinear expansion in $F_n$ requires an exponential number of macros (Theorem~\ref{thm:free-probabilistic}), in contrast with the logarithmically sparse macro set that suffices for exponential expansion in $A_n$.
This difference reflects underlying growth rates: balls grow polynomially in $A_n$ but exponentially in $F_n$.
Our study of macro sets in $A_n$ extends straightforwardly to the much larger class of free nilpotent monoids; they have nearly identical expansion properties to $A_n$ and can equally serve as models for HM according to the analysis presented here (see Sections~\ref{nilpotent_discussion} for details).

\begin{figure}[ht]
\centering
\begin{tikzpicture}[
    every node/.style={font=\small},
    stmt/.style={draw, rounded corners=4pt, inner sep=5pt},
    arr/.style={->, >=stealth, semithick},
    vnode/.style={fill=black, circle, inner sep=2pt},
]
\node[font=\small\bfseries] at (-3.5, 5) {FM DH};
\node[stmt] (ABC) at (-3.5, 3.8) {$A \wedge B \wedge C$};
\node[vnode] (v1) at (-5.5, 2) {};
\node[vnode] (v2) at (-1.5, 2) {};
\node[stmt] (AB) at (-7, 0) {$A \wedge B$};
\node[stmt] (C1) at (-4, 0) {$C$};
\node[stmt] (A) at (-3, 0) {$A$};
\node[stmt] (BC) at (0, 0) {$B \wedge C$};
\draw[arr] (ABC) -- (v1);
\draw[arr] (v1) -- (AB);
\draw[arr] (v1) -- (C1);
\draw[arr] (ABC) -- (v2);
\draw[arr] (v2) -- (A);
\draw[arr] (v2) -- (BC);
\node[font=\small\bfseries] at (6, 5) {MathLib DAG};
\node[stmt] (ABC2) at (6, 3.8) {$A \wedge B \wedge C$};
\node[stmt] (AB2) at (4.5, 0) {$A \wedge B$};
\node[stmt] (C2) at (7.5, 0) {$C$};
\draw[arr] (ABC2) -- (AB2);
\draw[arr] (ABC2) -- (C2);
\end{tikzpicture}
\caption{A FM DH fragment (left) and a corresponding MathLib DAG (right) for deriving $A \wedge B \wedge C$ via $\wedge$-introduction.
In the DH, filled dots represent hyperedges: each groups the premises used in a single inference step.
Both proofs (via $A \wedge B$ and via $B \wedge C$) are recorded.
The DAG selects one proof and replaces the hyperedge with ordinary edges.
If both proofs in the DH were recorded within a DAG, there would be an ambiguity in how a given conclusion followed from the premises.
Note also that the above example is special in the sense that the hyperedges are reversible.}
\label{fig:dh-vs-dag}
\end{figure}
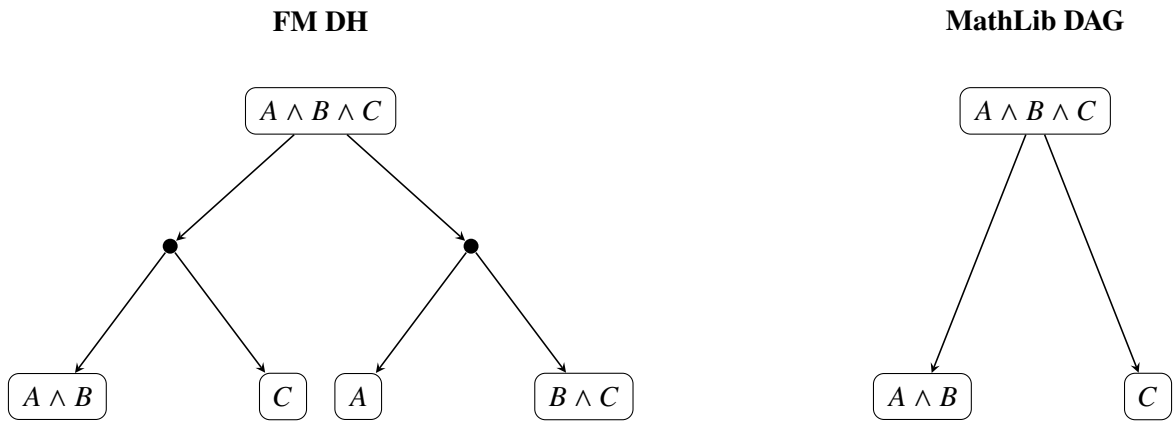

We test these facts against MathLib \cite{mathlib2020}, a large repository of mathematics written in Lean~4 \cite{Lean4} that contains hundreds of thousands of definitions, lemmas, and theorems.
We use MathLib as a proxy for HM.
MathLib can be viewed as a DAG extracted from the full deduction hypergraph (Fig.~\ref{fig:dh-vs-dag}).
Each MathLib element is a named subgraph of this DAG, rooted at the element itself and extending down to primitives.
Flattening this subgraph by recursively expanding all references yields a string of primitives.
The \emph{wrapped length} counts the tokens in an element's defining expression; 
the \emph{unwrapped length} counts the primitives in the flattened string.
We find the longest element, when fully unwrapped, reaches approximately $10^{104}$ primitive terms---Googol, the number, not the company.\footnote{\raggedright 
\href{https://github.com/leanprover-community/mathlib4/blob/d167cc6dc962ab340507362ea2f4bcfcff44f01b/Mathlib/AlgebraicGeometry/AffineTransitionLimit.lean\#L423}{\texttt{AlgebraicGeometry.Scheme.exists\_hom\_hom\_comp\_eq\_comp\_of\_locallyOfFiniteType}} in MathLib; see Section~\ref{mathlib_section} for discussion.}
Our primary observations about MathLib are as follows.
First, unwrapped length grows exponentially with depth (the longest path to primitives in the dependency graph).
Second, wrapped length is approximately constant across all depths.
Third, unwrapped length grows exponentially with wrapped length.
The MathLib data is consistent with the $A_n$ logarithmic-density regime and inconsistent with the $F_n$ alternatives we study. 

Our central inference is that HM is a thin subset of polynomial growth within the exponentially growing space FM.
This is a stronger claim than the observation that HM is a vanishingly small subset of FM: the latter would hold if both grew exponentially at different rates \cites{Shannon1948, CoverThomas2006}.
We propose $A_n$ as a model for HM and products $\prod F_{n_i}$ as a model for FM.
Since HM $\subset$ FM, the models must respect this inclusion; $A_n$ embeds into $\prod F_{n_i}$ by sending each generator to a distinct factor.

Our toy models map HM to monoids. 
What we observe clearly in both source and target is compression and hierarchical depth: in the monoid, both are deduced from a postulated macro set; in MathLib, both can be measured. 
The comparison allows us to infer properties of a hypothetical ``macro set'' for HM.
Physical models often gain power by defining abstractions not directly observed---the vector potential in electromagnetism, the Hilbert space in quantum mechanics. 
Similarly, our model may not be surjective: some abstractions in the monoid---notably the macro set itself---may have no direct counterpart on the HM side.
We do not identify a precise ``macro set'' within MathLib (or HM more generally) that maps to the macro set in the monoid, but regard this as a deep open problem, tantamount to locating the owner's manual for HM.

Place notation demonstrates two features of mathematics that our models can capture.
The first is hierarchy: the recursive fashion in which notation, ideas, definitions, and proofs are fitted together.
The second is parsimony: we have limited storage for new concepts, so definitions must be chosen to strike a balance between marking out landmarks in an infinite structure and not overtaxing our capacity to remember them.
HM works where compression is possible---it suits our minds and supports our inherent laziness, allowing large strides across the mathematical landscape with minimal effort.
Logarithmic density, as in powers of $10$, lies near this parsimony boundary.
The results of Section~\ref{monoid_models_section} explore the parsimony/expansion tradeoff systematically, showing how expansion rates depend on macro density across several regimes.
The MathLib data of Section~\ref{mathlib_section} confirms the hierarchical structure quantitatively: $\log(\text{unwrapped length})$ grows linearly with depth, with slope close to $1$ bit per level.

If compression characterizes human mathematics, it can also serve as a measure of mathematical \emph{interest}.
An element whose terse statement conceals an enormous proof body exhibits high \emph{deductive compression}; an element that compresses dramatically when definitions are applied sits in a region where the definitional hierarchy is useful. We call the latter \emph{reductive compression}.
Section~\ref{section_application_outlook} develops these ideas into quantitative interest measures and a PageRank-style refinement \cite{brin1998anatomy} that accounts for an element's role in supporting other high-value mathematics.
The goal is to give AI agents exploring formal mathematics a sense of direction: stay where compression is possible.

Various LLMs collaborated with us on the proofs of the theorems in Section~\ref{monoid_models_section}.
The \leanproof{} symbol below indicates that the theorem has been formally verified in Lean~4 by Aleph \cite{Aleph}, a theorem-proving system developed by Logical Intelligence.

The remainder of the paper is organized as follows.
Section~\ref{monoid_models_section} develops the monoid models and proves the main expansion theorems.
Section~\ref{mathlib_section} presents the MathLib analysis.
Section~\ref{discussion_section} further discusses the results and related ideas.
Section~\ref{section_application_outlook} considers future work on automating mathematical interest and related open questions.
Appendix~\ref{appendix:additional-theorems} contains additional expansion theorems for $A_n$.

\section{Monoid Models}
\label{monoid_models_section}

We study two basic monoids on $n$ generators $G = \{a_1, \ldots, a_n\}$: the free abelian monoid $A_n$ and the free monoid $F_n$.
In $A_n$, generators commute, so elements essentially live in $\N^n$ with componentwise addition.
In $F_n$, order matters and there are no relations; elements are finite strings over $G$.
For an element or word $w$ in either monoid, write $|w|_G$ for its \emph{length}: the sum of coefficients for $A_n$, or the string length for $F_n$.

A \emph{macro set} $M = \{g_i\}$ consists of additional generators, each defined by $g_i = w_i$ for some word $w_i$ written in terms of elements from $G$.
The augmented generating set is $G' = G \cup M$, and $|w|_{G'}$ denotes the minimum number of $G'$-generators needed to represent $w$.
Conceptually, while each $g_i \in M$ is an individual macro, the set $M$ itself represents a compression strategy.
\footnote{In logic and computer science, ``macro'' often refers to the transformation rule (the strategy) itself.
Here, we maintain the monoid-theoretic convention where the macro is the resulting element, and the set $M$ constitutes the strategy.}

We quantify the effectiveness of such a strategy with the \emph{expansion function},
\[
f_{G'}(s) = \sup\{r \in \N : B_G(r) \subseteq B_{G'}(s)\}.
\]
Here, the \emph{ball of radius $r$} is $B_G(r) = \{w : |w|_G \leq r\}$, with $B_{G'}(s)$ defined analogously.
Since $G \subseteq G'$, we have $|w|_{G'} \leq |w|_G$ and thus $B_G(s) \subseteq B_{G'}(s)$.
The expansion function measures the largest $G$-radius fully covered by the $G'$-ball of radius $s$.

Our main results are summarized in Table~\ref{tab:main-results}.
For concreteness, the table states the $A_n$ results for $A_1 = \N$; they extend to general $A_n$ by taking $n$ copies of each macro (one per generator), with the same asymptotic expansion rates.
In $A_n$, balls grow polynomially ($|B_G(r)| = \binom{r+n}{n}$), and sparse macros yield dramatic expansion---exponential or even infinite.
The polylogarithmic row reflects an upper bound (Theorem~\ref{thm:abelian-upper}); we do not establish a matching lower bound, so the true expansion for such macros may lie strictly between exponential and quasi-exponential.
In $F_n$, balls grow exponentially ($|B_G(r)| = \frac{n^{r+1}-1}{n-1}$), and expansion is linear for a polynomial-dense macro set and superlinear for an exponentially-growing macro set.

\begin{table}[ht]
\centering
\begin{tabular}{lllll}
\hline
Monoid & Macro $M$ & Density & Expansion $f_{G'}(s)$ & Theorem \\
\hline
$A_1$ & $\{m^k : m \geq 1\}$ & $r^{1/k}$ & $\infty$ & \ref{thm:abelian-waring} \\
$A_1$ & $\{b^{j^p} : j \geq 1\}$ & $(\log r)^{1/p}$ & $\leq e^{c \cdot s \log s}$ & \ref{thm:abelian-upper} \\
$A_1$ & $\{b^j : j \geq 1\}$ & $\log r$ & $\Theta(b^{cs})$ & \ref{thm:abelian-place} \\
$A_1$ & $\{b^{b^j} : j \geq 0\}$ & $\log_b \log r$ & $s^{b/(b-1)}$ to $s^{(2b-1)/(b-1)}$ & \ref{thm:abelian-double-log} \\$A_1$ & finite & $O(1)$ & $\Theta(s)$ & \ref{thm:abelian-finite} \\
\hline
$F_n$ & polynomial & $r^p$ & $O(s)$ & \ref{thm:free-polynomial} \\
$F_n$ & probabilistic & $n^r/\log r$ & $\geq e^{c\sqrt{s}}$ & \ref{thm:free-probabilistic} \\
\hline
\end{tabular}
\caption{Macro density versus expansion. For $A_n$ with $n > 1$, expansion rates remain the same. 
The expansion properties for $A_n$ also hold for free nilpotent monoids.
The $F_n$ polynomial result holds for any polynomial density and $n \geq 2$; up to constants, it holds for any finitely presented \ monoid of exponential growth. The $F_n$ probabilistic result gives a macro set of logarithmically-vanishing density ($|M \cap S_r|/|S_r| \sim 1/\log(r) \to 0$), where $S_r = \{w : |w|_G = r\}$.}
\label{tab:main-results}
\end{table}

We now proceed to the theorems that characterize expansion of different macro sets in $A_n$ and $F_n$.
The detailed proofs are not required for the later parts of the paper.

\subsection{Free Abelian Monoid}

Place notation is the archetypal example of compression in $A_n$.

\begin{theorem}[Place notation gives exponential expansion \leanproof{https://github.com/Aksenov239/lean-fun/blob/toy-math-model/LeanFun/theorem1.lean\#L1531}]\label{thm:abelian-place}
For $A_n$ and any integer $b \geq 2$, the macro set $M = \{ b^j a_i : i = 1, \ldots, n, \ j \geq 1 \}$ has logarithmic density and satisfies
\[
b^{s/(n(b-1)) - 1} \leq f_{G'}(s) \leq nb \cdot b^{s/(n(b-1))}
\]
for all integers $s \geq 1$.
In particular, $f_{G'}(s) = \Theta(b^{s/(n(b-1))})$.
\end{theorem}

\begin{proof}
The macro set $M = \{ g_{i,j} = b^j a_i : i = 1, \ldots, n, \ j \geq 1 \}$ has logarithmic density: the number of macros with $|g_{i,j}|_G = b^j \leq r$ is $n \lfloor \log_b r \rfloor = O(\log r)$.

\textbf{Lower bound.}
Any element $w \in A_n$ can be written uniquely as $w = x_1 a_1 + x_2 a_2 + \cdots + x_n a_n$ with $x_i \in \N$.
Writing each nonzero $x_i$ in base $b$ as $x_i = \sum_{j=0}^{J_i} c_{i,j} b^j$ with $c_{i,j} \in \{0, 1, \ldots, b-1\}$ and $J_i = \lfloor \log_b x_i \rfloor$, we have
\[
x_i a_i = c_{i,0} a_i + c_{i,1} g_{i,1} + c_{i,2} g_{i,2} + \cdots + c_{i,J_i} g_{i,J_i}.
\]
The $G'$-length of $x_i a_i$ is $\sum_{j=0}^{J_i} c_{i,j} \leq (b-1)(J_i + 1) = (b-1)(\lfloor \log_b x_i \rfloor + 1)$.
For $w \in B_G(r)$, we have $|w|_G = \sum_i x_i \leq r$, so each $x_i \leq r$ and thus
\[
|w|_{G'} \leq \sum_{i=1}^n (b-1)(\lfloor \log_b x_i \rfloor + 1) \leq n(b-1)(\log_b r + 1).
\]
Therefore $B_G(r) \subseteq B_{G'}(s)$ whenever $s \geq n(b-1)(\log_b r + 1)$, which gives $f_{G'}(s) \geq b^{s/(n(b-1)) - 1}$.

\textbf{Upper bound.}
We exhibit a hard-to-compress element.
For any integer $k \geq 1$, define $w_k = (b^k - 1)(a_1 + \cdots + a_n)$.
Then $|w_k|_G = n(b^k - 1)$.
Since $b^k - 1 = \sum_{j=0}^{k-1}(b-1)b^j$, $|w_k|_{G'} = n(b-1)k$.

Now given $s \geq 1$, choose $k = \lfloor s/(n(b-1)) \rfloor + 1$.
Then $|w_k|_{G'} = n(b-1)k > s$, so $w_k \notin B_{G'}(s)$.
Since $w_k \in B_G(n(b^k - 1)) \subseteq B_G(nb^k)$, we have $B_G(nb^k) \not\subseteq B_{G'}(s)$, and thus
\[
f_{G'}(s) < nb^k \leq nb \cdot b^{s/(n(b-1))}.
\]

Combining the bounds gives $f_{G'}(s) = \Theta(b^{s/(n(b-1))})$.
\end{proof}

We next establish an upper bound: with polylogarithmically many macros, 
expansion is at most quasi-exponential.

\begin{theorem}[Polylogarithmic density gives quasi-exponential expansion \leanproof{https://github.com/Aksenov239/lean-fun/blob/toy-math-model/LeanFun/theorem2.lean\#L1033}]\label{thm:abelian-upper}
For $A_n$, let $M \subseteq A_n$ be a macro set with polylogarithmic growth:
\[
|M \cap B_G(r)| \leq c(\log(e + r))^q \quad \text{for all } r \geq 0,
\]
for some constants $c, q > 0$.
Then there exists a constant $K > 0$ depending only on $n, c, q$ such that
\[
f_{G'}(s) \leq \exp(K s \log s) \quad \text{for all } s \geq 2.
\]
\end{theorem}

\begin{proof}
Fix $s \in \N$ and suppose $B_G(r) \subseteq B_{G'}(s)$, i.e., every element of length $\leq r$ can be expressed as a sum of at most $s$ generators from $G' = G \cup M$. We derive an upper bound on $r$ in terms of $s$.

\textbf{Step 1: Only macros of length $\leq r$ are relevant.}
Let $w \in B_G(r)$ and write $w = y_1 + \cdots + y_k$ with $k \leq s$ and $y_i \in G'$.
Each $y_i$ has length $|y_i|_G \geq 0$, and additivity of length in $\N^n$ gives $|w|_G = |y_1|_G + \cdots + |y_k|_G$.
Since $|w|_G \leq r$, all $|y_i|_G \leq r$; otherwise their sum would exceed $r$.
Thus, in any representation of elements of $B_G(r)$, only generators of length $\leq r$ can appear.

Define $M_r := M \cap B_G(r)$, so that $|M_r| \leq c(\log(e + r))^q$, and note that in every such representation each $y_i$ lies in $G \cup M_r$. Let
\[
t(r) := |G \cup M_r| = n + |M_r| \leq (n + c)(\log(e + r))^q.
\]

\textbf{Step 2: Upper bound on the number of words of length $\leq s$.}
The number of words of length $\leq s$ over an alphabet of size $t(r)$ is at most
\[
N_{\mathrm{words}}(r, s) := \sum_{k=0}^{s} t(r)^k \leq (s+1) (n+c)^s (\log(e + r))^{qs}.
\]
Each such word represents some element of $A_n$.
By our assumption $B_G(r) \subseteq B_{G'}(s)$ and Step 1, every element of $B_G(r)$ is representable by at least one such word.
Thus $|B_G(r)| \leq N_{\mathrm{words}}(r, s)$.

Since $|B_G(r)| = \binom{r+n}{n} \geq \frac{r^n}{n!}$, we have
\[
\frac{r^n}{n!} \leq (s+1) (n+c)^s (\log(e + r))^{qs}.
\]
Taking logarithms, for sufficiently large $s$:
\begin{equation}\label{eq:abelian-log-bound}
n \log r \leq (1 + \log(n+c)) s + qs \log\log(e + r).
\end{equation}

\textbf{Step 3: Bounding $r$.} We show that \eqref{eq:abelian-log-bound} fails for $K > 2q/n$ and sufficiently large $s$ whenever $\log r \geq Ks\log s$.
It suffices to consider $\log r = Ks\log s$, since larger $r$ only further violates the inequality.

For large $s$:
\[
\log\log(e + r) \leq \log 2 + \log(Ks \log s) \leq \log(2K) + 2 \log s.
\]
Substituting into \eqref{eq:abelian-log-bound} gives:
\[
nKs\log s \leq (1 + \log(n+c) + q\log(2K))s + 2qs\log s.
\]
Dividing by $s\log s$:
\[
nK \leq \frac{1 + \log(n+c) + q\log(2K)}{\log s} + 2q.
\]
For large $s$, the right-hand side approaches $2q$, so choosing $K > 2q/n$ yields a contradiction.

Thus $f_{G'}(s) < \exp(Ks\log s)$ for all sufficiently large $s$.
For small $s$, the bound in Step 2 shows $f_{G'}(s)$ is finite (since the polynomial growth in $r$ eventually beats the polylog growth in $r$), so by enlarging $K$ if necessary, the bound $f_{G'}(s) \leq \exp(Ks\log s)$ holds for all $s \geq 2$.
\end{proof}

Finally, we show that polynomial-density macros can yield infinite expansion.

\begin{theorem}[Polynomial density gives infinite expansion \leanproof{https://github.com/Aksenov239/lean-fun/blob/toy-math-model/LeanFun/theorem3.lean\#L381}]\label{thm:abelian-waring}
For any integer $k \geq 2$, there exists a macro set $M \subseteq A_n$ such that:
\[
|M \cap B_G(r)| \leq n r^{1/k} \quad \text{for all } r \geq 1,
\]
and
\[
f_{G'}(s) = \infty \quad \text{for all } s \geq ng(k),
\]
where $g(k)$ is the Waring constant (the smallest integer such that every nonnegative integer is a sum of at most $g(k)$ $k$-th powers).
\end{theorem}

\begin{proof}
For each generator $a_i \in G$ and each $m \in \N$, define the macro $g_{i,m} := m^k a_i$. Let
\[
M := \{ m^k a_i : i = 1, \ldots, n, \ m \geq 1 \}.
\]

\textbf{Growth bound.} A macro $m^k a_i$ has $G$-length $m^k$. The number of macros with $|g_{i,m}|_G = m^k \leq r$ is $\lfloor r^{1/k} \rfloor$ for each $i$, so
\[
|M \cap B_G(r)| = n \lfloor r^{1/k} \rfloor \leq n r^{1/k}.
\]

\textbf{Infinite expansion.} Any element $w \in A_n$ can be written as $w = x_1 a_1 + \cdots + x_n a_n$ with $x_i \in \N$. By Waring's theorem, each $x_i$ is a sum of at most $g(k)$ $k$-th powers:
\[
x_i = m_{i,1}^k + \cdots + m_{i,t_i}^k, \quad t_i \leq g(k).
\]
Thus
\[
x_i a_i = m_{i,1}^k a_i + \cdots + m_{i,t_i}^k a_i,
\]
where each term $m_{i,j}^k a_i$ lies in $M$. Summing over all $i$, the total number of macro terms is at most $\sum_{i=1}^n t_i \leq ng(k)$. Hence every element of $A_n$ lies in $B_{G'}(ng(k))$, giving $f_{G'}(ng(k)) = \infty$.
\end{proof}

\begin{remark}
For $k = 2$, Lagrange's four-square theorem gives $g(2) = 4$, so a macro set of squares with growth exponent $1/2$ achieves $f_{G'}(4) = \infty$.
\end{remark}

\begin{remark}
The macro sets in Theorems \ref{thm:abelian-place} and \ref{thm:abelian-waring} exhibit a sort of duality: 
Theorem \ref{thm:abelian-place} uses powers of a fixed base $\{b^j a_i : j \geq 1\}$, while Theorem \ref{thm:abelian-waring} uses fixed powers of varying bases $\{m^k a_i : m \geq 1\}$.
The sparser logarithmic-density set achieves exponential expansion; the polynomial-density set achieves infinite expansion. 
\end{remark}

Additional results for $A_n$, including the cases of double-logarithmic and finite macro density, appear in Appendix~\ref{appendix:additional-theorems}.

\subsection{Free Monoid}

In contrast to $A_n$, we now show that for $F_n$, a polynomially-growing macro set only achieves linear expansion and that superlinear expansion requires an exponentially-growing macro set. 
This reflects the exponential growth of the underlying monoid.

\begin{theorem}[Polynomial density gives linear expansion \leanproof{https://github.com/Aksenov239/lean-fun/blob/toy-math-model/LeanFun/theorem4.lean\#L1565}]\label{thm:free-polynomial}
For $F_n$ with $n \geq 2$, let $M$ be a macro set with at most $c\ell^p$ macros of each $G$-length $\ell \geq 2$, for some constants $c > 0$ and $p \geq 0$.
Then there exists a constant $d = d(n,p,c)$ such that for all integers $s \geq 1$:
\[
f_{G'}(s) < ds.
\]
Moreover, it suffices to choose an integer $d \geq 3$ satisfying:
\begin{equation}\label{eq:free-condition2}
n^d > 4e (n + c) d^{p+1}.
\end{equation}
\end{theorem}

\begin{proof}
Fix integers $r, s \geq 1$. Consider words of exact $G$-length $r$:
\[
S_r := \{w \in F_n : |w|_G = r\}, \qquad |S_r| = n^r.
\]
We will show that for an appropriate choice of $d$, $|\{w \in S_{ds} : |w|_{G'} \leq s\}| < n^{ds}$, which implies $B_G(ds) \not\subseteq B_{G'}(s)$.

Fix $r = ds$ and $1 \leq k \leq s$.
Since $F_n$ has no relations, any representation $w = y_1 \cdots y_k$ with $y_i \in G'$ and $|w|_G = ds$ is determined by:
\begin{enumerate}
\item A composition of $ds$ into $k$ positive parts $(\ell_1, \ldots, \ell_k)$, where $\ell_i = |y_i|_G$
\item A choice of generator in $G'$ of $G$-length $\ell_i$ for each $i$
\end{enumerate}
For each length $\ell \geq 1$, there are at most $(n+c)\ell^p$ generators in $G'$ of that length (exactly $n$ for $\ell = 1$, at most $c\ell^p$ for $\ell > 1$).
Therefore:
\[
|\{(y_1, \ldots, y_k) : y_i \in G', |y_i|_G = \ell_i\}| \leq (n+c)^k \prod_{i=1}^k \ell_i^p \leq (n+c)^k \left(\frac{ds}{k}\right)^{pk},
\]
where the right-most inequality follows from AM-GM with $\sum_{i=1}^k \ell_i = ds$:
\[
\prod_{i=1}^k \ell_i \leq \left(\frac{ds}{k}\right)^k.
\]
There are $\binom{ds-1}{k-1}$ such compositions, so:
\[
|\{w \in S_{ds} : |w|_{G'} = k\}| \leq \binom{ds-1}{k-1} (n+c)^k \left(\frac{ds}{k}\right)^{pk}.
\]
Summing over $k \leq s$:
\[
|\{w \in S_{ds} : |w|_{G'} \leq s\}| \leq \sum_{k=1}^s \binom{ds-1}{k-1} (n+c)^k \left(\frac{ds}{k}\right)^{pk} \leq \binom{ds-1}{s-1} \sum_{k=1}^s (n+c)^k \left(\frac{ds}{k}\right)^{pk}.
\]
The second inequality uses the fact that for $d \geq 2$ and $1 \leq k \leq s$, we have $\binom{ds-1}{k-1} \leq \binom{ds-1}{s-1}$. 

Define:
\[
\Sigma_s := \sum_{k=1}^s (n+c)^k \left(\frac{ds}{k}\right)^{pk}.
\]
Writing $b_k := \left((n+c)d^p \left(\frac{s}{k}\right)^p\right)^k$, one can verify that if $(n+c)d^p \geq e^p$, the sequence $b_k$ is increasing in $k$ for $1 \leq k \leq s$.
Since $d \geq 3$ and $n + c \geq 2$, we have $(n+c)d^p \geq 2 \cdot 3^p > e^p$, so the monotonicity condition holds.
Thus $\Sigma_s \leq s \cdot b_s = s \cdot ((n+c)d^p)^s \leq (2(n+c)d^p)^s$.
Using $\binom{ds-1}{s-1} < \binom{ds}{s} \leq (ed)^s$:
\[
|\{w \in S_{ds} : |w|_{G'} \leq s\}| \leq (ed)^s (2(n+c)d^p)^s = (2e(n+c)d^{p+1})^s.
\]
Choose $d$ such that $n^d > 4e(n+c) d^{p+1}$.
Then
\[
|\{w \in S_{ds} : |w|_{G'} \leq s\}| \leq (2e(n+c) d^{p+1})^s < (n^d/2)^s < n^{ds}.
\]
Therefore not all words of $G$-length $ds$ lie in $B_{G'}(s)$, so $f_{G'}(s) < ds$.
\end{proof}

\begin{remark}
When $p = 0$, condition \eqref{eq:free-condition2} simplifies to $n^d > 4e(n+c)d$, recovering the logarithmic-density case.
\end{remark}

The next result shows that a macro set of vanishing sphere density---though still exponentially large in absolute terms---allows for superlinear expansion.

\begin{theorem}[Probabilistic sparse macros give superlinear expansion in $F_n$ \leanproof{https://github.com/Aksenov239/lean-fun/blob/toy-math-model/LeanFun/theorem5.lean\#L3425}]
\label{thm:free-probabilistic}
Let $F_n$ be the free monoid on $n\ge 2$ generators $G=\{a_1,\dots,a_n\}$.
There exists a macro set $M\subset F_n$ such that
\[
\frac{|M\cap S_r|}{|S_r|}\longrightarrow 0
\qquad\text{as }r\to\infty,
\]
and
\[
\frac{f_{G'}(s)}{s}\longrightarrow \infty
\qquad\text{as }s\to\infty,
\]
where $S_r=\{w\in F_n:|w|_G=r\}$, $G' = G\cup M$, and $f_{G'}$ is the expansion function.
More quantitatively, there exist constants $K,c>0$ (depending only on $n$) such that
\[
B_G(r)\subseteq B_{G'}\!\bigl(K(\log r)^2\bigr)
\quad\text{for all sufficiently large }r,
\]
and hence
\[
 f_{G'}(s)\ \ge\ \exp\!\bigl(c\sqrt{s}\bigr)
\quad\text{for all sufficiently large }s.
\]
\end{theorem}

\begin{proof}
We build $M$ as a union of a small deterministic family of \emph{log-periodic} words
and an independent random family.

\medskip
\noindent\textbf{Step 0: log-periodic words.}
For a word $w=b_1\cdots b_L\in F_n$ of length $L=|w|_G$, say that $w$ has
\emph{period $d$} if $1\le d\le L$ and $b_j=b_{j+d}$ for all $1\le j\le L-d$.
Write $\mathrm{per}(w)$ for the least such $d$.

Fix constants $C>4\log n$ and $B\ge 2C$ (natural logarithms).
Define the deterministic macro family
\[
P \, :=\, \Bigl\{w\in F_n:\ \mathrm{per}(w)\le B\log\bigl(e+|w|_G\bigr)\Bigr\}.
\]

\medskip
\noindent\textbf{Step 1: random macros of density $1/\log \ell$.}
Independently for each word $u\in F_n$ of length $\ell\ge 2$, include $u$ in a random set $R$
with probability
\[
 p_\ell\ :=\ \frac{1}{\log(e+\ell)}.
\]
Let $M:=P\cup R$ and $G':=G\cup M$.

\medskip
\noindent\textbf{Step 2: vanishing sphere density.}
We have $|S_r|=n^r$.
First, count the deterministic part: any word in $P\cap S_r$ is determined by its period
$d\le B\log(e+r)$ and its first $d$ letters, hence
\[
|P\cap S_r|
\ \le\ \sum_{d\le B\log(e+r)} n^d
\ \le\ n^{B\log(e+r)+1}
\ =\ O\bigl(r^{B\log n}\bigr).
\]
So $|P\cap S_r|/n^r\to 0$.

Second, $|R\cap S_r|$ is $\mathrm{Bin}(n^r,p_r)$ with mean $\mu_r=n^r/\log(e+r)$.
A Chernoff bound gives
\[
\Pr\bigl(|R\cap S_r|\ge 2\mu_r\bigr)\ \le\ \exp(-\mu_r/3),
\]
and $\sum_r \exp(-\mu_r/3)<\infty$ since $\mu_r$ grows exponentially.
By Borel--Cantelli, almost surely $|R\cap S_r|\le 2n^r/\log(e+r)$ for all large $r$.
Thus, almost surely,
\[
\frac{|M\cap S_r|}{|S_r|}
\ \le\ \frac{|P\cap S_r|}{n^r}+\frac{|R\cap S_r|}{n^r}
\ \le\ o(1)+\frac{2}{\log(e+r)}
\ \longrightarrow\ 0.
\]

\medskip
\noindent\textbf{Step 3: a halving lemma (one macro consumes half the word).}
Fix a length $r$ and set
\[
 k(r)\ :=\ \lceil C\log(e+r)\rceil.
\]
For a word $w=b_1\cdots b_r\in S_r$, consider the family of long early substrings
\[
\mathcal{C}_r(w)\ :=\ \Bigl\{\, b_i b_{i+1}\cdots b_{i+\ell-1}\ :\
1\le i\le k(r),\ \lceil r/2\rceil\le \ell \le r-i+1 \,\Bigr\}.
\]
Each element of $\mathcal{C}_r(w)$ has length between $\lceil r/2\rceil$ and $r$.

\smallskip
\noindent\emph{Claim.}
For every $w\in S_r$, either $\mathcal{C}_r(w)$ contains an element of $P$,
or else all words in $\mathcal{C}_r(w)$ are pairwise distinct.

\smallskip
\noindent
Indeed, if two elements of $\mathcal{C}_r(w)$ were equal, they must have the same length $\ell$.
So we would have
\[
 b_i\cdots b_{i+\ell-1} \,=\, b_j\cdots b_{j+\ell-1}
\quad\text{for some }1\le i<j\le k(r).
\]
Let $d=j-i\le k(r)$.
Then the word $u=b_i\cdots b_{i+\ell+d-1}$ has period $d$, hence
\[
\mathrm{per}(u)\le d\le k(r)\le C\log(e+r)\le B\log\bigl(e+|u|_G\bigr),
\]
for $r$ large (since $|u|_G\asymp r$ and $B\ge 2C$).
Thus $u\in P$, and in particular $\mathcal{C}_r(w)$ contains an element of $P$.
This proves the claim.

Now suppose $\mathcal{C}_r(w)\cap P=\emptyset$.
Then the claim says $\mathcal{C}_r(w)$ is a set of distinct words.
Also, a crude count gives
\[
|\mathcal{C}_r(w)|
\ \ge\ k(r)\Bigl(\lfloor r/2\rfloor-k(r)\Bigr)
\ \ge\ \frac{r}{4}\,k(r)
\]
for all large $r$ (since $k(r)=O(\log r)$).
Because each $u\in \mathcal{C}_r(w)$ has length $\ge r/2$,
we have $p_{|u|_G}\ge 1/\log(e+r)$.
Independence of the random choice of $R$ across \emph{distinct} words gives
\[
\Pr\bigl(\mathcal{C}_r(w)\cap R=\emptyset\bigr)
\ =\ \prod_{u\in\mathcal{C}_r(w)} (1-p_{|u|_G})
\ \le\ \exp\!\Bigl(-\sum_{u\in\mathcal{C}_r(w)} p_{|u|_G}\Bigr)
\ \le\ \exp\!\Bigl(-\frac{|\mathcal{C}_r(w)|}{\log(e+r)}\Bigr).
\]
Using $|\mathcal{C}_r(w)|\ge \frac{r}{4}k(r)$ and $k(r)=\lceil C\log(e+r)\rceil$ gives
\[
\Pr\bigl(\mathcal{C}_r(w)\cap R=\emptyset\bigr)
\ \le\ \exp\!\Bigl(-\frac{C}{4}r\Bigr).
\]
Therefore, for every $w\in S_r$,
\[
\Pr\bigl(\mathcal{C}_r(w)\cap M=\emptyset\bigr)
\ \le\ \exp\!\Bigl(-\frac{C}{4}r\Bigr),
\]
since if $\mathcal{C}_r(w)$ meets $P$ then it certainly meets $M$.

By a union bound over all $w\in S_r$,
\[
\Pr\bigl(\exists\, w\in S_r \text{ with }\mathcal{C}_r(w)\cap M=\emptyset\bigr)
\ \le\ n^r \exp\!\Bigl(-\frac{C}{4}r\Bigr)
\ =\ \exp\bigl(r(\log n - C/4)\bigr).
\]
Since $C>4\log n$, the right-hand side is summable in $r$.
By Borel--Cantelli, with probability $1$ there is $r_0$ such that for all $r\ge r_0$
\emph{every} word $w\in S_r$ satisfies $\mathcal{C}_r(w)\cap M\neq\emptyset$.
Equivalently: for all large $r$, every length-$r$ word has a macro (deterministic or random)
of length $\ge r/2$ starting within its first $k(r)$ letters.

\medskip
\noindent\textbf{Step 4: recursive parsing and the $(\log r)^2$ bound.}
Fix $r\ge r_0$ and a word $w\in S_r$.
From Step~3, choose $1\le i\le k(r)$ and $\ell\ge r/2$ such that
\[
 u:=b_i\cdots b_{i+\ell-1}\in M.
\]
Then we can write
\[
 w \,=\, (b_1\cdots b_{i-1})\ \cdot\ u\ \cdot\ (b_{i+\ell}\cdots b_r),
\]
where the prefix $(b_1\cdots b_{i-1})$ uses $i-1\le k(r)$ generators from $G$ (fillers),
and $u$ is one macro.
The suffix has length at most $r-\ell \le r/2$.

Apply the same procedure to the suffix, and iterate.
After at most $\lceil \log_2 r\rceil$ iterations the remaining suffix has bounded length and can be
spelled out with generators.
At each iteration, we spend at most $k(r)+1 = O(\log r)$ tokens (fillers plus one macro).
Thus there is a constant $K$ such that every $w\in S_r$ has
\[
 |w|_{G'}\ \le\ K(\log r)^2
\]
for all $r$ large.
Since $|w|_{G'}$ is monotone in $|w|_G$, the same bound holds for every word of length $\le r$,
i.e.
\[
B_G(r)\subseteq B_{G'}\!\bigl(K(\log r)^2\bigr)
\quad\text{for all large }r.
\]
This gives the claimed expansion bound.
Writing $s=K(\log r)^2$ implies $r=\exp(\sqrt{s/K})$, so
$f_{G'}(s)\ge \exp(c\sqrt{s})$ with $c=1/\sqrt{K}$.
In particular $f_{G'}(s)/s\to\infty$.
\end{proof}

\begin{remark}
The macro set $M$ exists almost surely under the random inclusion process, but no explicit construction is given.
\end{remark}

\begin{remark}
The sphere density of $M$ satisfies $|M \cap S_r|/|S_r| \sim 2/\log r$.
In absolute terms, $|M \cap S_r| \sim 2n^r/\log r$: a vanishing fraction of each sphere, but exponentially many macros at each radius.
This contrasts with Theorem~\ref{thm:free-polynomial}, where at most $c\ell^p$ macros per length forces linear expansion.
\end{remark}

\section{Interpreting Data from MathLib}
\label{mathlib_section}

We now compare the results of Section~\ref{monoid_models_section} against MathLib.
We will think of MathLib as a proxy for HM.
Of course, MathLib's structure is partly shaped by Lean's type theory and by the choices of its contributor community.
Because MathLib is still small (roughly 500,000 elements in the version used in this
study\footnote{Commit \href{https://github.com/leanprover-community/mathlib4/tree/d167cc6dc962ab340507362ea2f4bcfcff44f01b}{ \texttt{d167cc6dc962ab340507362ea2f4bcfcff44f01b}}, dated 17 October
2025.}) and unevenly developed, we see prominent finite-size effects at its frontier.
This limits our ability to infer precise dimensional properties, but the data is consistent
with low-dimensional behavior more characteristic of $A_n$ than $F_n$.

\subsection{Constructing the dependency graph}

We construct a dependency graph from MathLib's internal Lean representation.
The vertices are all MathLib elements: lemmas, theorems, definitions, structures, and inductive types, plus a synthetic node for types (\leaninline{Sort}) and Lean core library elements (\leaninline{And}, \leaninline{Nat}, \leaninline{List}, etc.) that have no further MathLib dependencies.

Each MathLib element has two parts: a \emph{signature} (the theorem statement or type definition) and an optional \emph{body} (the proof or defining expression).
For each element $u$, we count how many times each other element $v$ is referenced, producing a directed edge from $u$ to $v$ weighted by this count.
Edges thus point toward dependencies.
The resulting data for each element is a vector of reference counts.
Note that this construction only records the multiplicity of each reference and forgets the order in which references appear within an element's expression tree.
For example, two proofs that apply the same lemmas with the same multiplicities but in different ways
produce identical dependency graph entries.

As a simple example, consider:
\begin{lstlisting}
lemma simple_lemma {A B : Prop} : A /\ B -> B :=
  fun h : A /\ B => And.right h
\end{lstlisting}
For reference, its internal representation is:
\begin{lstlisting}
type: forallE (A : Sort 0), forallE (B : Sort 0), 
      forallE (_ : app (app And #1) #0), #1
body: lam (A : Sort 0) => lam (B : Sort 0) => 
      lam (h : app (app And #1) #0) => app (app (app And.right #2) #1) #0
\end{lstlisting}
The signature contains two occurrences of \leaninline{Sort} (for the \leaninline{Prop} type of \leaninline{A} and \leaninline{B}) and one of \leaninline{And}.
The body contains \leaninline{Sort} twice, \leaninline{And} once, and \leaninline{And.right} once.
These produce weighted edges from \leaninline{simple_lemma} to the corresponding elements.

To extract these dependencies, we recursively traverse each MathLib element's signature and body, collecting every \leaninline{const} node (representing a reference to a named declaration):
\begin{lstlisting}
def collectElems (e : Expr) (acc : HashMap Name Nat := {}): HashMap Name Nat :=
  match e with
  | .const declName _             => acc.insert declName (acc.getD declName 0 + 1)
  | .app fn arg                   => collectElems arg (collectElems fn acc)
  | .lam _ binderType body _      => collectElems body (collectElems binderType acc)
  | .forallE _ binderType body _  => collectElems body (collectElems binderType acc)
  | .letE _ type value body _     => collectElems body (collectElems value (collectElems type acc))
  | .mdata _ expr                 => collectElems expr acc
  | .proj _ _ struct              => collectElems struct acc
  | .sort _                       => acc.insert (Name.mkSimple "Sort") (acc.getD (Name.mkSimple "Sort") 0 + 1)
  | _                             => acc
\end{lstlisting}
The key case is \leaninline{.const declName _}, which represents a reference to a named declaration; each such occurrence produces an edge in our graph.
The other cases (\leaninline{app} for function application, \leaninline{lam} for lambda abstraction, \leaninline{forallE} for dependent arrows, \leaninline{letE} for let-bindings) simply recurse into their subexpressions.
We also count \leaninline{.sort} occurrences, treating \leaninline{Sort} (which, for example, represents \leaninline{Prop}) as a primitive.

The resulting graph contains a small number of cycles (approximately 60 pairs of mutually dependent elements, all involving unsafe recursion).
Since our analysis requires an acyclic graph, we collapse each strongly connected component into a single vertex, reducing the vertex count from 463,719 to 463,661.

\subsection{Wrapped and unwrapped lengths, and depth}

Lean core library elements and \leaninline{Sort} (primitive elements) form the sinks of the MathLib DAG and are analogous to the generators $G$ in our monoid models.
We consider all other (non-primitive) elements to be analogous to the macro set $M$. 
Thus, we view MathLib as $G' = G \cup M$.

The \emph{unwrapped length} of an element $u$, denoted by $|u|_G$, is the total count of primitives when all references are recursively expanded. 
Primitives have unwrapped length $1$. 
For any non-primitive element with edges to elements $v_1, \ldots, v_k$ with weights $w_1, \ldots, w_k$:
\[
|u|_G = \sum_{i=1}^{k} w_i \cdot |v_i|_G.
\]
As the notation indicates, unwrapped length corresponds directly to the $G$-length in the monoid model.

The \emph{wrapped length} of an element $u$ is its token length: the number of tokens in its definition written in Lean as produced by the Lean parser.
Wrapped length corresponds to $|u|_{G' \setminus \{u\}}$ in the monoid model.
One could alternatively define wrapped length as the number of references in the internal Lean representation, i.e., the total weight of outgoing edges in the dependency graph.
However, tactics such as \leaninline{simp} and \leaninline{rw} expand during elaboration into many internal references to basic elements, inflating the reference count of some elements without introducing deep dependencies.
Under the reference-count metric, this produces elements with large wrapped length but slowly growing unwrapped length---a plateau in the compression curve that reflects proof automation rather than mathematical content.
Token count avoids this artifact: a tactic invocation is a single token regardless of how many references it generates internally.
Some MathLib elements are generated internally by Lean and have no human-written source; we omit these when reporting wrapped lengths.

The \emph{depth} of an element is the length of the longest path to primitives in the dependency DAG, i.e., the maximum number of successive reference-expansion steps required to reach generators. 
Primitives have depth $0$. 

\subsection{Distributions of wrapped and unwrapped lengths, and depth}

Figures~\ref{fig:histograms}(a)--(c) show the distributions of elements by
$\log_2(\text{unwrapped length})$, wrapped length, and depth.
In each case, the distribution is concentrated at low values with a long tail.
The declines at large values reflect finite-size effects: MathLib's coverage is uneven,
and we expect the tails to extend and fill in as the library grows.
A future analysis tracking MathLib's evolution over time could measure how the frontier expands.

\begin{figure}[ht]
\centering
\begin{subfigure}[b]{0.32\textwidth}
    \includegraphics[width=\textwidth]{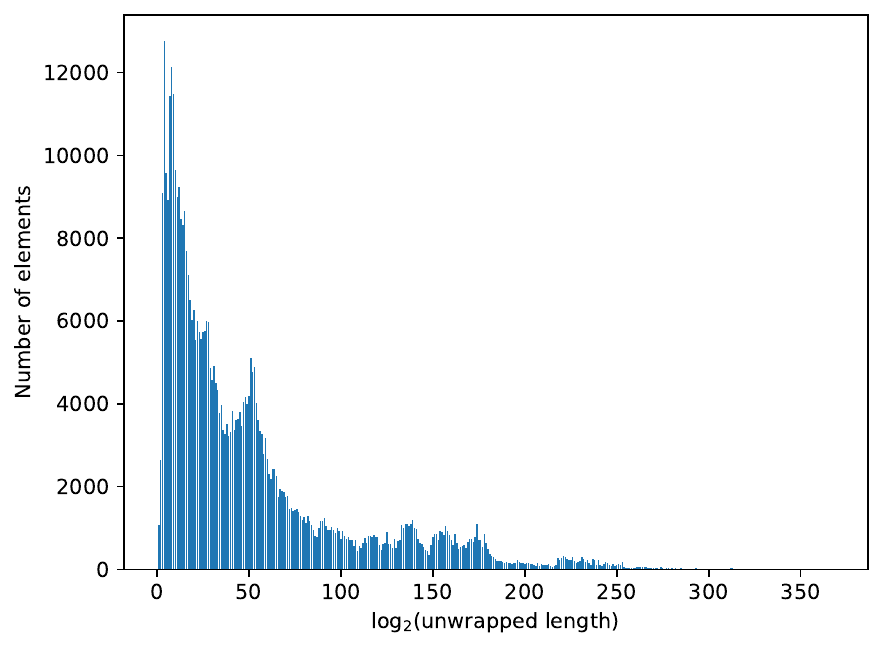}
    \caption{}
    \label{fig:hist-primitive}
\end{subfigure}
\hfill
\begin{subfigure}[b]{0.32\textwidth}
    \includegraphics[width=\textwidth]{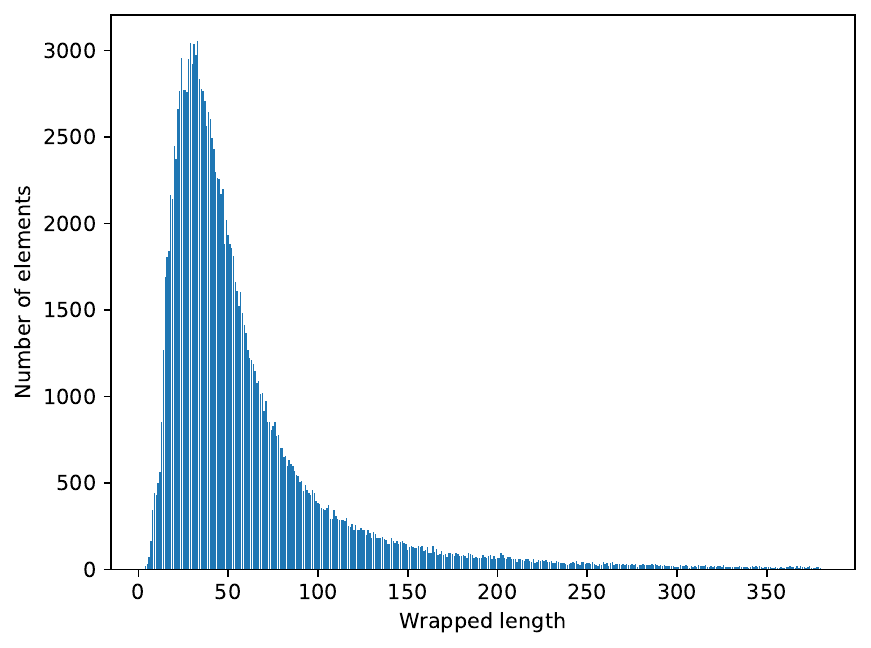}
    \caption{}
    \label{fig:hist-token}
\end{subfigure}
\hfill
\begin{subfigure}[b]{0.32\textwidth}
    \includegraphics[width=\textwidth]{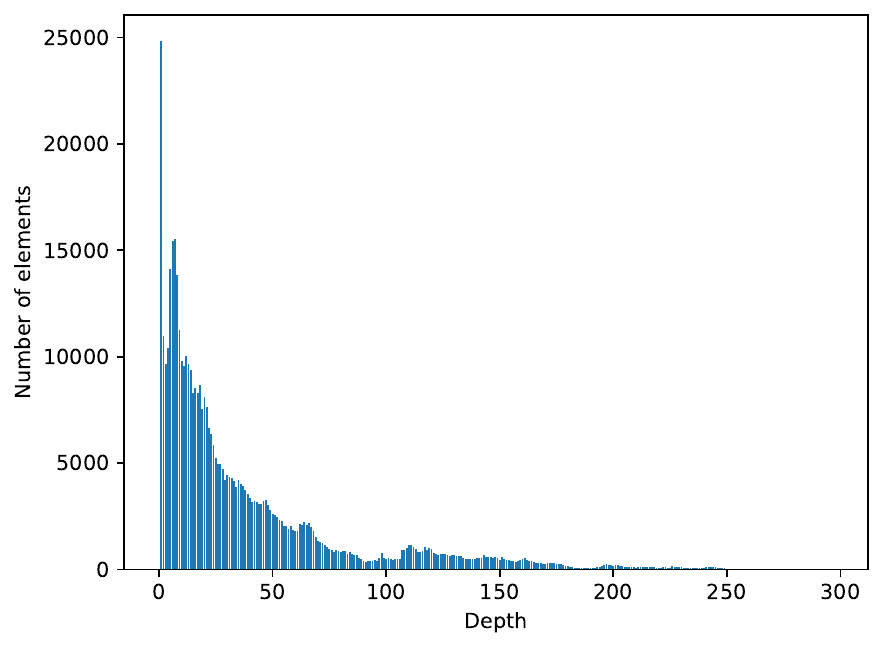}
    \caption{}
    \label{fig:hist-depth}
\end{subfigure}
\caption{Distributions of MathLib elements by (a) $\log_2(\text{unwrapped length})$,
(b) wrapped length, and (c) depth.}
\label{fig:histograms}
\end{figure}

\subsection{Unwrapped length versus wrapped length}
\label{sec:unwrapped-vs-wrapped}

Figure~\ref{fig:token-vs-primitive} shows median $\log_2(\text{unwrapped length})$ versus wrapped length. 
While $f_{G'}(s)$ measures the coverage of $G'$, these data points reflect the realized $G$-lengths $|u|_G$ achieved by MathLib's compression strategy.
The approximately linear relationship between $\log_2(\text{unwrapped length})$ and wrapped length (slope $\approx 0.4$ bits/token) indicates exponential expansion: each additional token yields a roughly constant multiplicative gain in primitive count. 
Further interpretation is provided in Section~\ref{sec:regime-discrimination}.
\begin{sloppypar}
Notice the spike at small wrapped length. 
This arises from the following types of elements, but not restricted to them: 
(1) abbreviations, such as \leaninline{isOpenMap_proj} from the module \href{https://github.com/leanprover-community/mathlib4/blob/master/Mathlib/Topology/VectorBundle/Basic.lean\#L674}{\leaninline{Topology.VectorBundle.Basic}} with wrapped length $7$ and unwrapped length $10^{48}$; 
(2) ``final'' theorems using complex intermediate ones, such as \leaninline{integrable} from the module \href{https://github.com/leanprover-community/mathlib4/blob/master/Mathlib/Analysis/Calculus/BumpFunction/Normed.lean\#L54}{\leaninline{Analysis.Calculus.BumpFunction.Normed}} with wrapped length $9$ and unwrapped length $2 \cdot 10^{54}$; 
(3) special cases of complex statements, such as \leaninline{infinitesimal_zero} from the module \href{https://github.com/leanprover-community/mathlib4/blob/master/Mathlib/Analysis/Real/Hyperreal.lean\#L731}{\leaninline{Analysis.Real.Hyperreal}} with wrapped length $8$ and unwrapped length $3.6 \cdot 10^{31}$. 
(Hyperlinks are located under the module name.)
\end{sloppypar}

\begin{figure}[ht]
\centering
\includegraphics[width=0.6\textwidth]{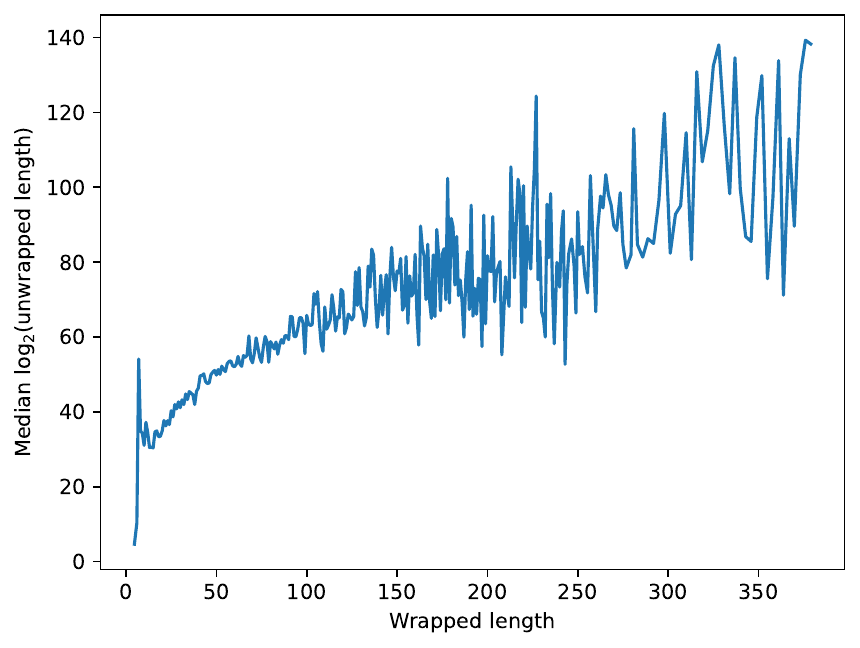}
\caption{Median $\log_2(\text{unwrapped length})$ versus wrapped length.
Each point represents the median over all elements at that wrapped length.}
\label{fig:token-vs-primitive}
\end{figure}

\subsection{Wrapped length versus depth}

Figure~\ref{fig:token-vs-depth} shows the median wrapped length as a function of depth.
The relationship is approximately flat (with at most a mild positive slope): (ignoring outliers at large depth) median wrapped length
hovers in the range of roughly 50--120 across depths 0--300, with no strong dependence on
depth.
This means that individual definitions do not become systematically longer as depth
increases, reminiscent of the powers of 10 macro of place notation.
This is not surprising since defining expressions are unlikely to become arbitrarily long; instead we expect modularization as wrapped length increases.

\begin{figure}[ht]
\centering
\includegraphics[width=0.6\textwidth]{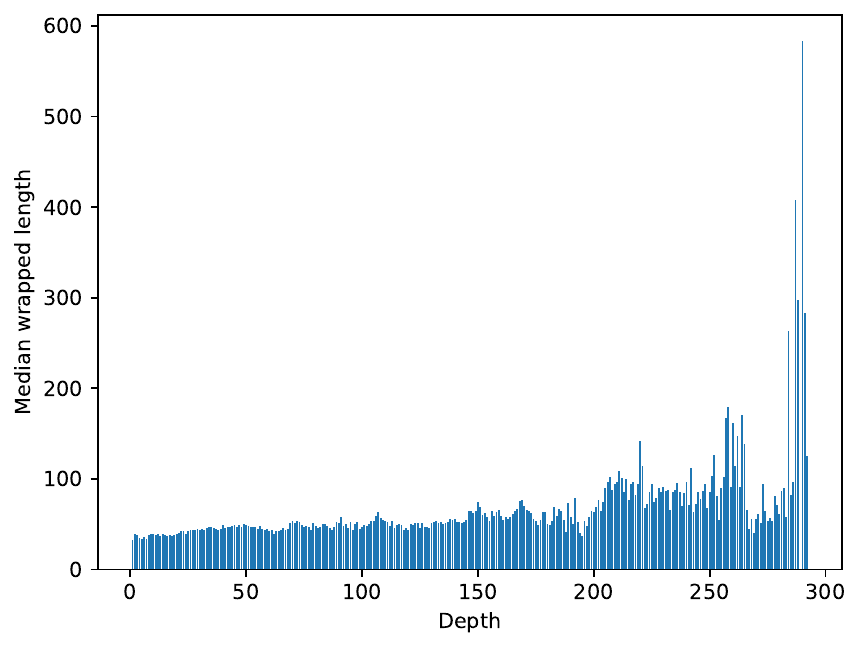}
\caption{Median wrapped length versus depth.}
\label{fig:token-vs-depth}
\end{figure}

\subsection{Unwrapped length versus depth}

Figure~\ref{fig:depth-vs-unwrapped} shows the median $\log_2(\text{unwrapped length})$ as a
function of depth.
The relationship is approximately linear with slope close to 1, indicating that unwrapped length grows exponentially with depth: each additional layer of definitions provides a roughly constant multiplicative gain in primitive count.
The maximum depth reaches approximately 300, consistent with the maximum unwrapped length of approximately $10^{104}$, achieved by the algebraic geometry entry noted in the introduction.

\begin{figure}[ht]
\centering
\includegraphics[width=0.6\textwidth]{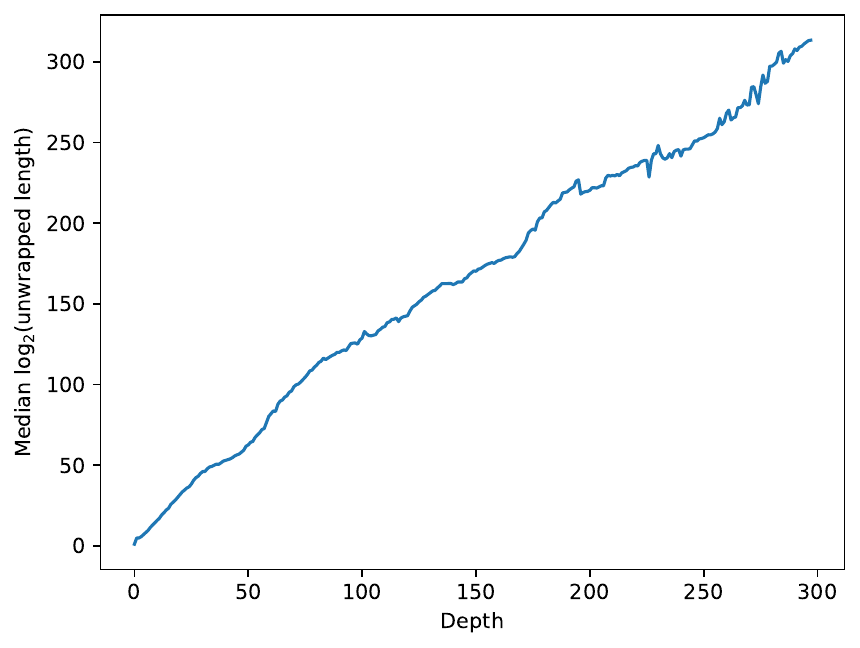}
\caption{Median $\log_2(\text{unwrapped length})$ versus depth.
The approximately linear relationship indicates exponential growth of unwrapped length with
depth.}
\label{fig:depth-vs-unwrapped}
\end{figure}

\subsection{Discriminating between regimes}
\label{sec:regime-discrimination}

To make contact between the MathLib measurements and the monoid models of Section~\ref{monoid_models_section}, we need to establish a correspondence between unwrapped length, wrapped length, depth, and their analogs in the monoid. 
Unwrapped length corresponds directly to the $G$-length $|u|_G$ in the monoid: both count the total number of primitive symbols after fully expanding all references. 

The \emph{wrapped length} of an element $u \in G'$ in the monoid is 
$|u|_{G' \setminus \{u\}}$: the minimum cost to represent $u$ using all 
generators in $G' = G \cup M$ except $u$ itself, where each generator 
(primitive or macro) contributes cost $1$. 
This mirrors the MathLib convention, where the wrapped length of an element is its token count---the cost of writing the definition using all 
available named elements.

The \emph{depth} of an element $u$ in the monoid generated by $G'$ is defined recursively using optimal representations.\footnote{For general monoids this definition will not be algorithmically computable due to lack of or limited cancellation; the definition is effective for $A_n$, $Nil_n$ (discussed later), and $F_n$.} 
Every primitive $g \in G$ has depth $0$. 
For any element $u$ with $u \notin G$, compute the optimal representation of $u$ in $G' \setminus \{u\}$; if this uses only primitives then $\mathrm{depth}(u) = 1$, and otherwise
\[
\mathrm{depth}(u) = 1 + \max\,\{\mathrm{depth}(v) : v \in M \setminus \{u\},\ 
v \text{ appears in the optimal representation of } u\}.
\]
Depth thus measures the length of the longest chain of macro dependencies required to optimally express $u$, bottoming out at primitives.
(Note that the expansion theorems of Section~\ref{monoid_models_section} impose only density conditions on the macro set.) 
This is analogous, but not identical to MathLib depth, which uses the longest path to primitives in the dependency DAG: 
MathLib depth reflects authorial choices, while monoid depth reflects the intrinsic hierarchical structure under optimal compression.

We note that we measured these quantities on all named elements in MathLib, not on all possible expressions that could be formed from them. 
Accordingly, in each monoid regime we restrict to the elements of the generating set $ G'= G \cup M$, rather than to all elements of the ambient monoid. 

Table~\ref{tab:regime-predictions} summarizes the relationships among the three quantities for each monoid regime.
The Parsimony column indicates whether the macro set grows at a strictly slower rate than the ambient monoid: subpolynomial macro growth in $A_n$, or subexponential macro growth in $F_n$.
Parsimonious macro sets achieve exponential expansion in $A_n$ but only linear expansion in $F_n$.
We discuss each row in turn below, specializing to $A_1 = \N$ and $F_2$ whenever no generality is lost.

\begin{table}[ht]
\centering
\small
\begin{tabular}{lllllc}
\hline
Regime & $\log|w|_G$ vs depth & $|w|_{G' \setminus \{w\}}$ vs depth & $\log|w|_G$ vs $|w|_{G' \setminus \{w\}}$ & Parsimony \\
\hline
$A_n$, log density & Linear & Flat$^\ast$ & Degenerate$^\ast$ & Yes \\
$A_n$, Waring & Linear & Flat & Degenerate & No \\
$A_n$, double-log & Exponential & Doubly exp. & Logarithmic & Yes \\
$F_n$, polynomial & Degenerate & Degenerate & Logarithmic & Yes \\
$F_n$, probabilistic & Linear & Quadratic & Concave ($\sqrt{\cdot}$) & No \\
\hline
\end{tabular}
\caption{Predicted relationships among $\log|w|_G$ (log unwrapped length), $|w|_{G' \setminus \{w\}}$ (wrapped length), and depth for each monoid regime. 
Measurements are restricted to elements of 
$G'= G \cup M$. 
``Degenerate'' means the independent variable (i.e., the ``x'' in ``y vs x'') is bounded.
The asterisk ($\ast$) indicates that if a generic element of the monoid is considered then Flat $\rightarrow$ Linear and Degenerate $\rightarrow$ Linear; see \ref{mathlib_summary} for discussion.}
\label{tab:regime-predictions}
\end{table}

\subsubsection{$A_n$, logarithmic-density macros (Theorem~\ref{thm:abelian-place})}
The macro set is $M = \{b^j a_i : i = 1, \ldots, n,\ j \geq 1\}$.
We specialize to $A_1$ with $M = \{b^j : j \geq 1\}$.
Macro $b^j$ has $G$-length $b^j$, so $\log|b^j|_G = j \log b$.
Its optimal representation in $G' \setminus \{b^j\}$ uses $b$ copies of $b^{j-1}$, giving wrapped length $b$, independent of $j$.
Iterating this, $\mathrm{depth}(b^j) = j$.
Thus, $\log|w|_G$ is linear in depth, wrapped length is flat across all depths, and since wrapped length is constant while $\log|w|_G$ varies freely, the third relationship is degenerate.

If one considers generic elements of $A_1$ rather than restricting to macro elements, the picture changes for columns 2 and 3. 
A generic element $x \in \mathbb{N}$ has a wrapped length $\sim \log x$ and depth $\sim \log x$. 
Thus wrapped length generally grows 
linearly with depth, and $\log|x|_G \sim |x|_{G' \setminus \{x}$. 

\subsubsection{$A_n$, polynomial-density macros / Waring 
(Theorem~\ref{thm:abelian-waring})}
The macro set is $M = \{m^k : m \geq 1\}$ for fixed $k \geq 2$ (specializing to $A_1$). 
Macro $m^k$ has $G$-length $m^k$, so $\log|m^k|_G = k \log m$.
By Waring's theorem, $m^k$ can be written as a sum of at most $g(k)$ $k$-th powers of strictly smaller integers, so wrapped length is at most $g(k)$, independent of $m$.
We don't know the optimal decomposition at each stage, but can bound the depth: a halving strategy (expressing $m^k$ using $k$-th powers of integers $\approx m/2$, then recursing) gives depth $O(\log m)$ and a linear relationship in column 1; slower reductions give greater depth and a sublinear relationship, down to logarithmic if depth $\sim m$.
In all cases, wrapped length remains bounded, so columns 2 and 3 are flat and degenerate respectively.

The Waring and log-density regimes produce identical predictions in the first three columns of Table~\ref{tab:regime-predictions} but differ in parsimony.
In contrast to the log-density case, for generic elements of $A_1$ in the Waring regime, column 3 remains degenerate since every $x \in \mathbb{N}$ has wrapped length at most $g(k)$.

\subsubsection{$A_n$, double-logarithmic density 
(Theorem~\ref{thm:abelian-double-log})}
The macro set is $M = \{b^{b^j} : j \geq 0\}$ (specializing to $A_1$). 
Macro $m_j = b^{b^j}$ has $\log|m_j|_G = b^j \log b$, which grows exponentially in $j = \mathrm{depth}(m_j)$.
The optimal representation of $m_j$ in $G' \setminus \{m_j\}$ uses 
$\lfloor m_j / m_{j-1} \rfloor = b^{b^{j-1}(b-1)}$ copies of $m_{j-1}$, so wrapped length $\sim b^{b^{j-1}(b-1)}$, which grows doubly exponentially in 
depth.
Eliminating $j$: $\log(\text{wrapped}) \sim b^{j-1}(b-1)\log b \sim 
\frac{b-1}{b} \log|m_j|_G$, so $\log|m_j|_G \sim \frac{b}{b-1} 
\log(\text{wrapped})$, giving a logarithmic relationship in column 3.
The first three columns are inconsistent with the MathLib data.

\subsubsection{$F_n$, polynomial-density macros 
(Theorem~\ref{thm:free-polynomial})}
The macro set has at most $c\ell^p$ elements of each $G$-length $\ell$, out of 
$n^\ell$ total words of that length in $F_n$. The fraction of words at length 
$\ell$ that are macros is therefore $c\ell^p / n^\ell$, which vanishes 
exponentially fast. For any macro $m$ with $|m|_G = r$, the probability that 
its optimal representation in $G' \setminus \{m\}$ contains another macro is 
negligible: the exponentially sparse macro set cannot populate the exponentially 
growing spheres of $F_n$ densely enough to sustain hierarchical nesting. 
Consequently, essentially all macros have depth $1$, with their optimal representations consisting almost entirely of primitives. 
Thus, $\log|m|_G$ vs depth and wrapped vs 
depth are degenerate. 
Since $\log|m|_G \approx \log(\text{wrapped})$, the column 3 entry is logarithmic. 
Again, the first three column entries are inconsistent with MathLib.

\subsubsection{$F_n$, probabilistic sparse macros 
(Theorem~\ref{thm:free-probabilistic})}
Here the macro set has $\sim 2n^r / \log r$ elements at radius $r$: exponentially many, though a logarithmically vanishing fraction ($|M \cap S_r|/|S_r| \sim 2/\log r$) of the sphere.
This contrasts with the polynomial case, where the fraction vanishes exponentially.
The absolute density is sufficient for the halving scheme of Theorem~\ref{thm:free-probabilistic} to work: (nearly) every word of length $r$ has a macro of length $\geq r/2$ starting within its first $k(r) = O(\log r)$ positions. 
Hierarchical depth develops as a result.

For a macro $m$ with $|m|_G = r$, the halving scheme gives 
$\mathrm{depth}(m) \sim \log_2 r$, so $\log|m|_G \sim \mathrm{depth}$. 
For the wrapped length: at each of the $\sim \log_2 r$ levels, we spend at most $k(r_t) = O(\log r_t)$ primitive fillers plus one macro, where $r_t \leq r/2^t$ is the remaining length at level $t$. The total wrapped length 
is
\[
\text{wrapped} \leq \sum_{t=0}^{\log_2 r} O(\log(r/2^t)) 
= \sum_{t=0}^{\log_2 r} O(\log r - t) = O((\log r)^2),
\]
where the sum is arithmetic with $O(\log r)$ terms each of size $O(\log r)$. 
Since $\mathrm{depth} \sim \log r$, wrapped $= O(\mathrm{depth}^2)$. 
Eliminating depth: $\log|m|_G \sim \mathrm{depth} \sim 
\sqrt{\text{wrapped}}$, giving a concave ($\sqrt{\cdot}$) relationship in column 3. 

\subsubsection{Summary and Identifying the ``Macro Set''}\label{mathlib_summary}

The MathLib data (Figures~\ref{fig:token-vs-primitive}, 
\ref{fig:token-vs-depth}, and \ref{fig:depth-vs-unwrapped}) shows: an approximately linear relationship in column 1, an approximately flat relationship in column 2, and an approximately linear relationship in column 3. 
The $F_n$ regimes are inconsistent with the data in at least two of the three columns.
The $A_n$ log-density and Waring regimes both match columns 1 and 2 for macro elements, but predict a degenerate column 3.
For generic elements in the log-density regime (recall the asterisks in Table~\ref{tab:regime-predictions}), columns 2 and 3 both become linear.
The MathLib column 2 plot is approximately flat but may have a mild positive slope, consistent with a mixture of macro and generic elements in the log-density regime.
This points to $A_n$ with log-density macros, which we note is also the parsimonious regime.

We do not take the all-non-primitives identification seriously as the true ``macro set'' for MathLib.
MathLib contains abbreviations and trivial specializations (see Section~\ref{sec:unwrapped-vs-wrapped}) that merely invoke deep elements, contributing little compression on their own.
On the other hand, some elements contribute substantial compression, e.g., the filter abstraction that unifies many limit theorems into a single framework.
The identification of the correct ``macro set'' is a central problem.

One simple approach to refining the ``macro set'' is to filter by in-degree in the dependency graph, removing elements in the bottom $x$ percentile for varying $x$. 
Removing an element from $M$ leaves unwrapped lengths unchanged, but increases wrapped lengths and decreases depths. 
Another option is to restrict the ``macro set'' to definition-like elements, e.g., those whose resulting type is \leaninline{Sort}; we found that the resulting macro set provides little compression.
One could also formulate the problem as an optimization, e.g., given a dependency DAG, find the ``macro set'' of fixed size $k$ that minimizes the total wrapped length.
Whether these or related approaches bring the three metrics into better agreement with the $A_n$ log-density predictions is an open question.

The difficulty of identifying the ``macro set'' may reflect limitations of MathLib itself as a proxy for HM.
An enriched representation that captures relationships MathLib leaves implicit (e.g., that a family of related theorems are instances of a single pattern), or permits (as a hypergraph does) storage of multiple proofs of the same theorem, might admit a more satisfactory ``macro set.''
Finding a representation with an identifiable ``macro set'' is not merely a question of library science: it could reveal new mathematical structure and help direct automated agents.

\section{Discussion}\label{discussion_section}

\subsection{Why monoids?}

Monoids model the sequential structure of proofs; for the dependency structure, where multiple premises are consumed simultaneously, $n$-categories and ``globular magmas'' offer alternative frameworks.
Inspired by the ``growth of groups'' \cite{Milnor1968}, which is essentially independent of generating set, one might expect that the choice of formal system should not greatly affect the shape of mathematics.
However, the recursive nature of mathematics leads to unexpectedly large distortions: rewriting between different formulations of a mathematical theory with identical proof power can be non-recursively 
inefficient (see \cite{BDF} and references therein).
If rules and syntax are fixed, the underlying minimal-proof DAG is a discrete metric space with the axioms as base point.
This metric geometry is modeled by the Cayley-graph geometry of each monoid and underlies the concepts of polynomial and exponential growth.

\subsection{Why monoids rather than groups}

Groups introduce inverses, and inverses complicate the expansion analysis.
This difficulty has historical precedent: Post \cite{Post} encoded the halting problem into the word problem for monoids in 1947, but extending this to groups required another decade of work by Boone and Novikov \cite{Novikov1955, Boone1959}.
The gap reflects the technical complications that inverses introduce.

For our purposes, the problem is concrete.
In the free group, any word $w$ can be written with length 2 by choosing macros $m$ and $m'$ that nearly cancel: $mm' = w$.
By making $m$ and $m'$ sufficiently long, we can do this for every word while keeping the macro density arbitrarily small.
This ``cancellation trick'' trivializes the expansion question for groups.

One might object that deduction rules like Modus Ponens ($A \to B$, $A \vdash B$) involve a kind of cancellation.
However, reversible logic is universal with only modest overhead \cite{Bennett1973}, so cancellation is not essential to computation.
More pragmatically, free monoids admit exact analysis---we stand where the light is good.

While on the topic of groups, there is a Lie-theoretic analogy worth noting: the inclusion $A_n \hookrightarrow \prod_{i=1}^n F_{n_i}$ resembles the inclusion of a maximal torus $T^{n-1} \hookrightarrow SO(n)$.
If $A_n$ is broadened to free-nilpotent monoids $Nil_{n,k}$ ($n$ counting homological rank and $k$ the nilpotency level), then the inclusion $Nil_{n,k} \hookrightarrow \prod_{i=1}^n F_{n_i}$ is reminiscent of the Iwasawa decomposition $G = KAN$ of a semisimple 
Lie group.

\subsection{Nilpotent and solvable monoids}
\label{nilpotent_discussion}

The dichotomy between $A_n$ and $F_n$ reflects different growth rates: polynomial versus exponential.
This difference, not abelian versus non-abelian, is what matters.

The nonnegative Heisenberg monoid illustrates the polynomial-growth case.\footnote{The nonnegative Heisenberg monoid consists of $3 \times 3$ nonnegative integer matrices with zeros below the diagonal and ones on the diagonal.
It can be presented as $\langle a, b, z \mid ab = baz, az = za, bz = zb \rangle$.}
It is nilpotent and non-abelian, and of polynomial growth.
Here again, and in all free nilpotent monoids, logarithmic-density macros yield exponential expansion, just as for $A_n$.
The proof is simple.
$Nil_{n,k}$ is the monoid with $n$ ``base'' generators, with additional secondary generators to simulate commutation, and the minimal relations needed to encode the level-$k$ right-nested commutators (recall we have no inverses in these monoids).
All these generators, and their number is polynomial in $n$ and $k$, individually generate a copy of $\N$.
Build separate macro sets in each of these $\N$-directions.
Then take the union of these $\N$-macro-sets to be the macro set for $Nil_{n,k}$.
Up to polynomial factors it will have the same expansion function as the mini-macro-sets in each $\N$.
Among monoids with cancellation (meaning $ac = bc$ implies $a = b$, and $ca = cb$ implies $a = b$), the geometrically crucial property of ``polynomial growth'' is characterized (as in groups) by Gromov's criterion \cite{Gromov1981}: they are sub-monoids of finite index within a nilpotent monoid.
So there is at hand a well-understood class of slow-growth monoids to help us model HM.

The nonnegative sector of the SOL lattice illustrates the exponential-growth case.
This monoid is solvable and non-abelian---more structured than a free monoid---but still has exponential growth.
Counting arguments parallel to Theorem~\ref{thm:free-polynomial} show that it admits no poly-logarithmically growing macro with super-linear expansion.

\subsection{Why formal mathematics resists compression}

That human mathematics compresses well is familiar from experience and confirmed by our MathLib analysis.
Less obvious is that FM, as a whole, must contain vast regions that resist compression.

By compression we mean something specific: reductive compression, the local substitution of definitions, not arbitrary algorithmic encoding.
The digits of $\pi$ illustrate the distinction.
The number $\pi$ has low Kolmogorov complexity---a short program computes it---but the digit string admits no known local compression via pattern substitution.
Reductive compression requires finding a repeated structure to name.
High reductive compressibility does, however, imply low 
Kolmogorov complexity, since the DAG of definitions encodes a short reconstruction procedure.
Moreover, this procedure runs in at most linear time, so reductively compressible elements have low logical depth in the sense of Bennett~\cite{BennettDepth}.
Thus reductive compressibility is a more restrictive condition than low Kolmogorov complexity.

With this understanding, the incompressibility of most of FM follows from standard complexity assumptions.
Consider theorems of the form ``the Boolean formula $S$ is unsatisfiable.'' 
We argue that for typical $S$ such theorems lie outside HM.
Such statements are easy to write down, but for generic $S$ the shortest proofs are believed to be exponentially long---essentially, one must check all possible variable assignments.
No system of definitions can shortcut this exhaustive search.
If such proofs could be radically compressed, this would contradict the assumption $P \neq NP$.

Thus HM consists of the compressible regions of FM---the mathematics where definitions provide leverage.
This motivates our search for monoid models that are highly compressible (modeling HM) yet embed in larger monoids that resist compression (modeling FM).
Compressibility may even serve as a first approximation to mathematical taste; Section~\ref{section_application_outlook} develops 
this idea, distinguishing \emph{reductive} compression (shortening statements via definitions) from \emph{deductive} compression (shortening proofs given statements) and using PageRank to combine both into an automated measure of mathematical interest.
 
\subsection{Comparison with cellular automata}

Kolmogorov complexity anticipates the action of any possible algorithm 
and thus has no locality restriction, whereas reductive compression requires the algorithm to act locally.
A cellular automaton (CA) also acts locally, but on a fixed lattice or graph.
In reductive compression, the underlying lattice also undergoes collapsing at each step (as well as changing its site labels); this kind of dynamic may be called a ``collapsing cellular automaton'' (CCA), and is how we think of compression operationally.

\subsection{A self-referential remark}

In writing this paper, we have introduced a new mathematical concept: the expansion function $f_{G'}(s)$.
This generalizes a notion from additive number theory.
The \emph{additive rank} of a subset $S \subseteq \N$ is the fewest copies of $S$ whose sumset equals $\N$; the \emph{asymptotic rank} is the fewest copies whose sumset is cofinal in $\N$.
When $S$ is too sparse for any finite number of copies to cover $\N$, these notions break down.
Our expansion function measures instead how large a ball can be covered by sums of at most $s$ elements from $G'$---a natural generalization when infinite coverage is impossible.

That we were led to introduce a definition while studying the role of definitions in mathematics is perhaps fitting.
Definition formation is so natural to mathematical practice that even analyzing math requires new definitions.

\section{Application and Outlook}\label{section_application_outlook}

Can we give AI agents a sense of direction—an automated criterion for which mathematical statements merit attention?
Surely historical and cultural factors influence human judgments of mathematical interest, but here we attempt to identify observables intrinsic to the mathematical representations themselves (see also \cite{BDF}). 

\begin{table}[ht]
\centering
\begin{tabular}{lcc}
\hline
 & $G$ & $G' \setminus \{u\}$ \\
\hline
$S$ & $|S|_G$ & $|S|_{G' \setminus \{u\}}$ \\
$B$ & $|B|_G$ & $|B|_{G' \setminus \{u\}}$ \\
\hline
\end{tabular}
\caption{Four measures of an element with signature $S$ (statement) and body $B$ (proof).
The ratio across rows measures reductive compression; the ratio down the $G' \setminus \{u\}$ column measures deductive compression.}
\label{tab:compression-quadrant}
\end{table}

The central thesis of this paper suggests compressibility as a natural candidate.
Here we consider two types.
For any element $u$ in a mathematical corpus, whether a definition, lemma, or theorem, define the \emph{reductive compression}:
\[
T_0(u) = \frac{|S|_G+|B|_G}{|S|_{G' \setminus \{u\}}+|B|_{G' \setminus \{u\}}},
\]
the ratio of unwrapped to wrapped length for the full element (signature plus body).
Here, we are using the monoid notation defined in Section~\ref{sec:regime-discrimination}, where $|\cdot|_G$ corresponds to unwrapped length and $|\cdot|_{G' \setminus \{u\}}$ to wrapped length.
Table~\ref{tab:compression-quadrant} summarizes the four resulting measures of an element.
Elements with large $T_0$ (``taste'') live in regions where definitions provide substantial leverage---precisely the regions we have identified with human mathematics.
An agent exploring mathematics can track the average value of $T_0$ as it explores from region to region.
This might assist it in staying close to HM.
Although this risks biasing the agent toward abstraction---which should not be pursued for its own sake---$T_0$ can contribute to an agent's sense of direction.

The ratio of wrapped body length to wrapped signature length measures another type of compressibility, \emph{deductive compression}\footnote{An alternative is $|B|_G/|S|_G$, i.e., measured in $G$ rather than $G' \setminus \{u\}$, which measures the primitive proof-to-statement ratio before compression. This would strongly reward theorems with naive statements (like Fermat's Last Theorem) that require elaborate additional developments (the theory of elliptic curves) for their proof.}:
\[
I_0(u) = \frac{|B|_{G' \setminus \{u\}}}{|S|_{G' \setminus \{u\}}}.
\]
Elements without bodies (primitives, structure declarations, inductive types) have $I_0 = 0$; they may still achieve high $T_0$ using definitions.
Elements with large $I_0$ (``interest'') have short statements but long proofs, even after compression.
In MathLib, elements with high $I_0$ include the deep theorems of algebraic geometry and category theory, where layers of abstraction compress enormous unwrapped expressions into manageable statements.
One imagines that formalized versions of landmark results---Fermat's Last Theorem, the Poincar\'e conjecture, resolution of singularities---would achieve exceptional compression ratios, their terse statements belying vast proof machinery.

When statements become long enough to encode logical conundra, $I_0$ can be gamed: metamathematical constructions produce arbitrarily large values for elements of questionable interest.
For example (and with thanks to  Sam Buss), the theorem asserting $k$-consistency: that a formal system has no proof of $0=1$ in fewer than $k$ symbols, can have tiny 
compressed length, since recursive function theory allows the rapid description of certain enormously large  integers k. 
However Pudl\'ak~\cite{Pudlak1986, Pudlak1987} showed that any proof  of k-consistency in a sufficiently rich (and consistent) system must have length $\Omega(k^{1/2})$---no system of definitions can shortcut it. By taking $k = \mathrm{BB}(n)$, BB denoting the Busy Beaver function, one obtains a family of theorems with phenomenally large $I_0 \approx \text{BB}^{1/2}(k)/\log k$  that few would consider \textit{that} interesting, at least on an individual basis, since this is merely one of a huge family, parameterized by $k$, of similar theorems; individually, they are logical curiosities rather than core mathematics.


\subsection{A PageRank-style refinement}
The issue is that $I_0$ treats all high-compression elements equally, regardless of their role in the broader mathematical structure.
A refinement should incorporate not just the compression achieved by an element, but also its \emph{usefulness} in building other high-value elements.

Google's PageRank algorithm~\cite{brin1998anatomy} offers a natural framework.
Consider the full dependency graph, with edges pointing from each element to its dependencies.
A random walk on this graph would accumulate at primitives---the sinks of the DAG---which achieve low compression and are not mathematically interesting in the sense we seek.
The standard fix is teleportation: at each step, with probability $\alpha$ the walker follows an edge, and with probability $1 - \alpha$ it jumps to a random node.
Even with uniform teleportation, this may produce nontrivial rankings by identifying useful elements from graph structure alone.
We suggest biasing teleportation toward high-compression elements.
We parametrically combine our two compression measures into $J_0 = \beta T_0 + (1-\beta) I_0$ (after normalizing each to comparable scales) for some $0 < \beta < 1$, and let an element $u$ be chosen as the teleportation destination with probability $J_0(u) / \sum_v J_0(v)$.
The resulting transition matrix is
\[
P(v, u) = \alpha \cdot \frac{w(u, v)}{W(u)} + (1 - \alpha) \cdot \frac{J_0(v)}{Z},
\]
where $w(u,v)$ is the number of times $u$ references $v$, $W(u) = \sum_x w(u,x)$ is the total reference count of $u$, and $Z = \sum_x J_0(x)$.

A stationary distribution $\pi$ satisfies $\sum_u P(v, u) \pi(u) = \pi(v)$; standard PageRank theory (i.e., the Perron--Frobenius theorem) guarantees existence and uniqueness since every node has positive teleportation probability.
Define
\[
I_1(u) = \pi(u).
\]
Elements score highly if they are either high-compression themselves (frequent teleportation destinations) or depended upon by elements that are visited often.
This captures ``load-bearing'' elements: those that support the compressible regions of mathematics.
We expect that $\alpha$ will need to be properly tuned to avoid trivial $\pi$, such as when all mass is concentrated at either the axioms or the largest $J_0$ elements in the DAG.

\subsection{Some Open Questions}

First, the computational challenge: determining optimal compression requires searching over possible definitions, which is computationally expensive.
Our interest measures assume a fixed set of definitions, but an agent exploring FM would need to propose new definitions on the fly.
Can this be done efficiently enough to guide exploration?

Second, definitional compression occupies one extreme of a spectrum.
At the other extreme is Kolmogorov complexity, which allows arbitrary algorithmic compression but is uncomputable.
Definitional compression is local and efficiently verifiable: applying a definition requires only checking that certain properties have been derived, and the process runs in at most linear time (in Bennett's~\cite{BennettDepth} terminology, it has low ``logical depth'').
Is there useful middle ground—compression methods more powerful than local substitution but still computationally tractable?

Finally, we note an empirical question for MathLib and similar repositories: Do the proofs of ``interesting'' statements stay close to the ground (using only shallow intermediate lemmas), or do they take flight through highly compressed intermediate statements?
In physical terms, what potential barriers must be overcome to reach deep theorems?
The depth and mass distributions in Section~\ref{mathlib_section} offer preliminary data, but a systematic study correlating these metrics with human judgments of interest remains to be done.

\bibliography{refs.bib}

\appendix

\section{Additional Expansion Theorems}
\label{appendix:additional-theorems}

The following results complete the picture for $A_1 = \N$ summarized in Table~\ref{tab:main-results}.

\begin{theorem}[Double-logarithmic density gives polynomial expansion \leanproof{https://github.com/Aksenov239/lean-fun/blob/toy-math-model/LeanFun/theorem6.lean\#L1701}]\label{thm:abelian-double-log}
For $A_1 = \mathbb{N}$ and any integer $b \geq 2$, the macro set $M = \{ b^{b^j} : j \geq 0 \}$ has double-logarithmic density and satisfies
\[
c_1 \, s^{b/(b-1)} \leq f_{G'}(s) \leq c_2 \, s^{(2b-1)/(b-1)}
\]
for all $s \geq 1$, where $c_1, c_2 > 0$ depend only on $b$.
\end{theorem}

\begin{proof}
Let $m_j = b^{b^j}$ for $j \geq 0$.
The number of macros with $G$-length at most $r$ is the number of $j$ with $b^{b^j} \leq r$, i.e., $j \leq \log_b \log_b r$.
Thus $M$ has double-logarithmic density.

\textbf{Greedy representation of $m_k - 1$.}
The largest macro not exceeding $m_k - 1$ is $m_{k-1}$.
The number of copies used is
\[
\left\lfloor \frac{m_k - 1}{m_{k-1}} \right\rfloor = \left\lfloor b^{b^{k-1}(b-1)} - b^{-b^{k-1}} \right\rfloor = b^{b^{k-1}(b-1)} - 1,
\]
with remainder $m_{k-1} - 1$.
Letting $T_k = |m_k - 1|_{G'}$, we obtain
\[
T_k = (b^{b^{k-1}(b-1)} - 1) + T_{k-1}, \qquad T_0 = b - 1.
\]
The first term dominates: $T_k \sim b^{b^{k-1}(b-1)}$.

\textbf{The elements $m_k - 1$ are hardest to compress.}
Let $S_j = \max\{|x|_{G'} : x < m_j\}$.
We claim $S_j = T_j$ for all $j \geq 0$.
The base case $S_0 = T_0 = b - 1$ is clear.
For the inductive step, suppose $S_j = T_j$ and consider any $x$ with $m_j \leq x < m_{j+1}$.
Writing $x = c_j \cdot m_j + r$ with $0 \leq r < m_j$, we have $|x|_{G'} \leq c_j + T_j$.
Since $c_j \leq b^{b^j(b-1)} - 1$, we obtain $S_{j+1} \leq T_{j+1}$.
Equality holds at $x = m_{j+1} - 1$.

\textbf{Upper bound.}
Fix $s \geq b - 1$ and let $k$ satisfy $T_k \leq s < T_{k+1}$.
The element $(s - T_k + 1) m_k + (m_k - 1)$ has $G'$-length equal to $(s - T_k + 1) + T_k = s + 1$, so it is not in $B_{G'}(s)$.
Thus $f_{G'}(s) < (s + 2) m_k$.
From the recurrence, $T_k \geq b^{b^{k-1}(b-1)} - 1$, so $T_k \leq s$ implies $b^{b^{k-1}(b-1)} \leq s + 1 \leq 2s$.
Thus, $b^{k-1} \leq \frac{\log_b(2s)}{b-1}$ and we have $m_k = b^{b \cdot b^{k-1}} \leq (2s)^{b/(b-1)}$.
Therefore $f_{G'}(s) \leq c_2 \, s^{(2b-1)/(b-1)}$.
For $s < b - 1$, we simply choose $c_2$ sufficiently large.

\textbf{Lower bound.}
With $k$ as above, any $x = c \cdot m_k + y$ with $0 \leq y < m_k$ satisfies $|x|_{G'} \leq c + T_k$.
If $c \leq s - T_k$ then $x \in B_{G'}(s)$, so $f_{G'}(s) \geq (s - T_k + 1) m_k - 1$.

\textit{Case $s \leq 2T_k$.}
Then $T_k \geq s/2$, so $m_k = b^{b^k} \geq c \, T_k^{b/(b-1)} \geq c \, (s/2)^{b/(b-1)} \geq c_1 \, s^{b/(b-1)}$.
Since $s - T_k + 1 \geq 1$, we obtain $f_{G'}(s) \geq c \, s^{b/(b-1)}$.

\textit{Case $s > 2T_k$.}
Then $s - T_k + 1 \geq s/2$.
From $s < T_{k+1} \leq C b^{b^k(b-1)}$ we obtain $b^k \geq \frac{\log_b(s/C)}{b-1}$ and thus $m_k = b^{b^k} \geq (s/C)^{1/(b-1)}$.
Therefore $f_{G'}(s) \geq (s/2)(s/C)^{1/(b-1)} \geq c_1 \, s^{b/(b-1)}$.
\end{proof}

\begin{theorem}[Finite macro gives linear expansion \leanproof{https://github.com/Aksenov239/lean-fun/blob/toy-math-model/LeanFun/theorem7.lean\#L69}]\label{thm:abelian-finite}
For $A_1 = \mathbb{N}$, let $M$ be a finite macro set.
Then $f_{G'}(s) = \Theta(s)$.
\end{theorem}

\begin{proof}
Let $L = \max\{|m|_G : m \in M\}$ be the largest $G$-length among all macros.

\textbf{Upper bound.}
Any element $x$ with $|x|_{G'} \leq s$ is a sum of at most $s$ generators from $G' = G \cup M$.
Each generator has $G$-length at most $L$, so $|x|_G \leq sL$.
Thus $B_{G'}(s) \subseteq B_G(sL)$, which gives $f_{G'}(s) \leq sL$.

\textbf{Lower bound.}
Since $G \subseteq G'$, we have $|x|_{G'} \leq |x|_G$ for all $x$.
Thus $B_G(s) \subseteq B_{G'}(s)$, which gives $f_{G'}(s) \geq s$.

Combining the bounds gives $f_{G'}(s) = \Theta(s)$.
\end{proof}

\begin{remark}
Theorems~\ref{thm:abelian-double-log} and~\ref{thm:abelian-finite} extend to $A_n$ with the macro sets $\{b^{b^j} a_i : i = 1, \ldots, n, j \geq 0\}$ and any finite $M \subseteq A_n$ respectively, yielding the same asymptotic expansion rates.
\end{remark}

\end{document}